\documentclass[11pt]{article}
\usepackage{complexity}

\usepackage[english]{babel}
\usepackage[utf8x]{inputenc}
\usepackage[T1]{fontenc}

\usepackage[margin=1.in,marginparwidth=1.75cm]{geometry}

\usepackage[algo2e, ruled, vlined]{algorithm2e}
\usepackage{amsmath}
\usepackage{amsfonts}
\usepackage{graphicx}
\usepackage{amsthm}
\usepackage{dsfont}
\usepackage{bbm}

\usepackage{amssymb}

\usepackage{enumitem}

\usepackage[colorlinks=true, allcolors=blue]{hyperref}

\newtheorem{theorem}{Theorem}[section]
\newtheorem{claim}[theorem]{Claim}

\newtheorem{lemma}[theorem]{Lemma}
\newtheorem{example}[theorem]{Example}
\newtheorem{corollary}[theorem]{Corollary}

\newtheorem{remark}[theorem]{Remark}
\newenvironment{reminder}[1]{\bigskip
	\noindent {\bf Reminder of #1.}\em}{\smallskip}

\newtheorem{definition}[theorem]{Definition}

\newcommand{\Lap}{\mathbf{Lap}}
\newcommand{\calA}{\mathcal{A}}
\newcommand{\AAA}{\mathcal{A}}
\newcommand{\DDD}{\mathcal{D}}

\newcommand{\BBB}{\mathcal{B}}
\newcommand{\N}{\mathbb{N}}

\newcommand{\calD}{\mathcal{D}}

\newcommand{\calX}{{\mathcal{X}}}
\newcommand{\calY}{{\mathcal{Y}}}

\newcommand{\reals}{\mathbb{R}}
\newcommand{\eps}{\varepsilon}

\newcommand{\calS}{\mathcal{S}}
\newcommand{\SSS}{\mathcal{S}}

\newcommand{\Geom}{\mathtt{Geom}}
\newcommand{\supp}{\mathsf{supp}}

\newcommand{\ignore}[1]{}

\title{\~{O}ptimal Differentially Private Learning of Thresholds and Quasi-Concave Optimization}

\author{{\normalfont Edith Cohen}\thanks{Google Research and Tel Aviv University. \texttt{edith@cohenwang.com}.} \and Xin Lyu\thanks{UC Berkeley and Google Research. \texttt{lyuxin1999@gmail.com}.} \and Jelani Nelson\thanks{UC Berkeley and Google Research. \texttt{minilek@alum.mit.edu}.} \and Tam\'{a}s Sarl\'{o}s\thanks{Google Research. \texttt{stamas@google.com}.} \and Uri Stemmer\thanks{Tel Aviv University and Google Research. \texttt{u@uri.co.il}.}}

\date{November 11, 2022}

\begin{document}

\maketitle

\begin{abstract}
    The problem of learning threshold functions is a fundamental one in machine learning. Classical learning theory implies sample complexity of $O(\xi^{-1} \log(1/\beta))$ (for generalization error $\xi$ with confidence $1-\beta$).  The private version of the problem, however, is more challenging and in particular, the sample complexity must depend on the size $|X|$ of the domain. Progress on quantifying this dependence, via lower and upper bounds, was made in a line of works over the past decade.  In this paper, we finally close the gap for approximate-DP and provide a nearly tight upper bound of $\widetilde{O}(\log^* |X|)$, which matches a lower bound by Alon et al (that applies even with improper learning) and improves over a prior upper bound of $\widetilde{O}((\log^* |X|)^{1.5})$ by Kaplan et al. We also provide matching upper and lower bounds of $\tilde{\Theta}(2^{\log^*|X|})$ for the additive error of private quasi-concave optimization (a related and more general problem).  Our improvement is achieved via the novel Reorder-Slice-Compute paradigm for private data analysis which we believe will have further applications.
\end{abstract}

\section{Introduction}

Motivated by the large applicability of learning algorithms to settings involving personal individual information, Kasiviswanathan et al.~\cite{KLNRS08} introduced the model of {\em private learning} as a combination of {\em probably approximately correct (PAC)} learning \cite{valiant1984theory,vapnik71uniform} and {\em differential privacy} \cite{DMNS06}.
For our purposes, we can think of a (non-private) learner as an algorithm that operates on a {\em training set} containing
labeled random examples (from some distribution over some domain $X$), and outputs a hypothesis $h$ that %   
misclassifies fresh examples with probability at most (say) $\tfrac{1}{10}$. 
It is assumed that the ``true'' classification rule, which is unknown to the learner, is taken from a (known) class $C$ of possible classification rules, where intuitively, learning becomes ``harder'' as the class $C$ becomes ``richer''.  
A {\em private learner} must achieve the same goal while guaranteeing that the choice of $h$ preserves {\em differential privacy} of the training set. This means that the
choice of $h$ should not be significantly affected by any particular labeled example in the training set. Formally, the definition of differential privacy is as follows.
\begin{definition}[\cite{DMNS06}]\label{def:DP}
Let $\AAA:X^*\rightarrow Y$ be a randomized algorithm whose input is a dataset $D\in X^*$. Algorithm $\AAA$ is {\em $(\eps,\delta)$-differentially private (DP)} if for any two datasets $D,D'$ that differ on one point (such datasets are called {\em neighboring}) and for any outcome set $F\subseteq Y$ it holds that
$
\Pr[\AAA(D)\in F]\leq e^{\eps}\cdot\Pr[\AAA(D')\in F]+\delta.
$
\end{definition}

Since its inception, research on the private learning model has largely focused on understanding the {\em amount of data} that is needed in order to achieve both the privacy and the utility goals simultaneously (a.k.a.\ the {\em sample complexity of private learning}). The holy grail in this line of research is to come up with a (meaningful) combinatorial measure that given a class $C$ characterizes the sample complexity of privately learning $C$. However, after almost 15 years of intensive research, this question is still far from being well-understood. At a high level, works on the sample complexity of private learning can be partitioned into two meta approaches:

\begin{enumerate}
    \item[{\bf 1.}] {\bf Deriving generic upper and lower bounds (as a function of the class ${\boldsymbol{C}}$).} This avenue has produced several fascinating results, that relate the sample complexity of private learning to the
    {\em Littlestone dimension} of the class $C$, a combinatorial dimension that is known to characterize {\em online learnability} (non-privately) \cite{AlonBLMM22}. However, the resulting bounds are {\em extremely} loose (exhibiting a tower-like gap between them). Furthermore, it is known that, in general, this is the best possible in terms of the Littlestone dimension alone.
    
    \item[{\bf 2.}] {\bf Focusing on specific test-cases, squeezing them until the end to reveal structure.} 
    This avenue has produced several fascinating techniques that has found many applications, even beyond the scope of private learning. 
    Arguably, the most well-studied test-case is that of {\em one dimensional threshold functions}, where the class $C$ contains all functions that evaluate to 1 on a {\em prefix} of the (totally ordered) domain $X$.\footnote{
    Let $X\subseteq\reals$. A threshold function $f$ over $X$ is specified by an element $u\in X$ so that $f(x)=1$ if $x\leq u$ and $f(x)=0$ for $x>u$. In the corresponding learning problem, we are given a dataset containing labeled points from $X$ (sampled from some unknown distribution $\DDD$ over $X$ and labeled by some unknown threshold function $f^*$), and our goal is to output a hypothesis $h:X\rightarrow\{0,1\}$ such that ${\rm error}_{\DDD}(h,f^*)\triangleq\Pr_{x\sim\DDD}[h(x)\neq f^*(x)]$ is small.
    } Even though this class is trivial to learn without privacy considerations, in the private setting it is surprisingly complex.  The sample complexity of privately learning threshold functions has been studied in a sequence of works~\cite{BKN10,ChaudhuriH11,BeimelNS13,doi:10.1137/140991844,DBLP:conf/focs/BunNSV15,BunDRS18,DBLP:conf/stoc/AlonLMM19,KaplanLMNS20}, producing many interesting tools and techniques that are applicable much more broadly. 
\end{enumerate}

In this work we present new tools and proof techniques that allow us to obtain a {\em tight} upper bound on the sample complexity of privately learning threshold functions (up to lower order terms). %  
This concludes a long line of research on this problem. 
In addition, we present matching upper and lower bounds for the related problem of {\em private quasi-concave optimization}. Before presenting our new results, we survey some of the progress that has been made on these questions.

\subsection{On our current understanding of privately learning threshold functions}

Early works on the sample complexity of  private learning focused on the case where the privacy parameter $\delta$ is set to zero, known as the {\em pure private} setting. While this significantly limits the applicability of the model, the pure-private setting is often much easier to analyze. Indeed, already in the initial work on private learning, Kasiviswanathan et al.~\cite{KLNRS08} presented a generic bound of $O(\log|C|)$ on the sample complexity of learning a class $C$ with pure privacy.\footnote{To simplify the exposition, in the introduction we omit the dependency of the sample complexity in the utility and privacy parameters.} This implies an upper bound of $O(\log|X|)$ on the sample complexity of privately learning threshold functions over an ordered domain $X$ (because $|C|=|X|$ for this class). Beimel et al.~\cite{BKN10} presented a matching lower bound for {\em proper} pure-private learners (these are learners whose output hypothesis must itself be a threshold function). Feldman and Xiao \cite{doi:10.1137/140991844} then showed that this lower bound also holds for pure-private {\em improper} learners. 

The sample complexity of privately learning thresholds in the more general setting, where $\delta$ is not restricted to be zero (known as {\em approximate privacy}), was studied by Beimel at al.~\cite{BeimelNS13}, who showed an improved upper bound of $\tilde{O}\left(8^{\log^*|X|}\right)$ on the sample complexity. This is a dramatic improvement in asymptotic terms over the pure-private sample complexity (which is $\Theta(\log|X|)$), coming tantalizingly close to the non-private sample complexity of this problem (which is constant, independent of $|X|$). Interestingly, to obtain this result, Beimel at al.~\cite{BeimelNS13} introduced a tool for {\em privately optimizing quasi-concave functions} (to be surveyed next), a generic tool which has since found many other applications. Bun et al.~\cite{DBLP:conf/focs/BunNSV15} then presented a different approximate-private learner with improved sample complexity of $\tilde{O}\left(2^{\log^*|X|}\right)$, and another different construction with similar sample complexity was presented by \cite{BunDRS18}. 
Bun et al.~\cite{DBLP:conf/focs/BunNSV15} also showed a lower bound of $\Omega(\log^*|X|)$ that holds for any (approximate) private {\em proper}-learner for thresholds. Alon  et al.~\cite{DBLP:conf/stoc/AlonLMM19} then proved a lower bound of $\Omega(\log^*|X|)$ that holds even for {\em improper} learners for thresholds. Finally, a recent work of Kaplan et al.~\cite{KaplanLMNS20} presented an improved algorithm with sample complexity $\tilde{O}((\log^*|X|)^{1.5})$. 

To summarize, our current understanding of the task of privately learning thresholds places its sample complexity somewhere between $\Omega(\log^*|X|)$ and $\tilde{O}((\log^*|X|)^{1.5})$.

\subsection{Privately optimizing quasi-concave functions}
Towards obtaining their upper bound for privately learning thresholds, Beimel et al.~\cite{BeimelNS13} defined a family of optimization
problems, called {\em quasi-concave optimization problems}. The possible solutions are ordered,
and quasi-concavity means that if two solutions $x\leq z$ have quality of at least $q$, then any solution
$x\leq y\leq z$ also has quality of at least $q$. The optimization goal is to find a solution with (approximately) maximal quality. Beimel et al.~\cite{BeimelNS13} presented a private algorithm for optimizing such problems, guaranteeing additive error at most $\tilde{O}\left(8^{\log^*T}\right)$, where $T$ is the number of possible solutions. They observed that the task of learning thresholds can be stated as a quasi-concave optimization problem, and that this yields a private algorithm for thresholds over a domain $X$ with sample complexity $\tilde{O}\left(8^{\log^*|X|}\right)$. 
Since the work of Beimel et al.~\cite{BeimelNS13}, quasi-concave optimization was used as an important component for designing private algorithms for several other problems, including geometric problems~\cite{BeimelMNS19,GaoS21}, clustering~\cite{NissimSV16,feldman2017coresets}, and privately learning halfspaces~\cite{BeimelMNS19,DBLP:conf/nips/KaplanMST20}. 

We stress that later works on privately learning thresholds (following \cite{BeimelNS13}) did not present improved tools for quasi-concave optimization (instead they worked directly on learning thresholds). As quasi-concave optimization generalizes the task of learning thresholds (properly), the lower bound of \cite{DBLP:conf/focs/BunNSV15} also yields a lower bound of $\Omega(\log^*T)$ on the additive error of private algorithms for quasi-concave optimization. That is, our current understanding of private quasi-concave optimization places its additive error somewhere between $\Omega(\log^*T)$ and $\tilde{O}\left(8^{\log^*T}\right)$. 
An improved upper bound would imply improved algorithms for all of the aforementioned applications, and a stronger lower bound would mean an inherent limitation of the algorithmic techniques used in these papers.

\subsection{Our contributions}
Our main result is presenting a private algorithm for learning thresholds, with {\em optimal} sample complexity (up to lower order terms):
\begin{theorem}[Informal version of Theorem~\ref{theo:learning-threshold}]\label{thm:IPPintro}
There is an approximate private algorithm for (properly) learning threshold functions over an ordered domain $X$ with sample complexity $\widetilde{O}(\log^* |X|)$.
\end{theorem}

This improves over the previous upper bound of $\widetilde{O}\left((\log^* |X|)^{1.5}\right)$ by \cite{KaplanLMNS20}, and matches the lower bound of $\Omega(\log^*|X|)$ by \cite{DBLP:conf/focs/BunNSV15,DBLP:conf/stoc/AlonLMM19} (up to lower order terms). This concludes a long line of research aimed at understanding the sample complexity of this basic problem. A key to our improvement is a novel paradigm, which we refer to as the {\em Reorder-Slice-Compute} paradigm (to be surveyed next), allowing us to simplify both the algorithm and the analysis of \cite{KaplanLMNS20}.

Inspired by our simplified algorithm for thresholds, we design a new algorithm for private quasi-concave optimization with an improved error of $\tilde{O}\left(2^{\log^* T}\right)$, a polynomial improvement over the previous upper bound of  $\tilde{O}\left(8^{\log^* T}\right)$ by \cite{BeimelNS13}. 

\begin{theorem}[Informal version of Theorem~\ref{theo:quasi-concave-upper-bound}]\label{thm:upperIntro}
There exists an approximate-private algorithm for quasi-concave optimization with additive error $\widetilde{O}(2^{\log^*T})$, where $T$ is the number of possible solutions.
\end{theorem}

As we mentioned, this immediately translates to improved algorithms for all of the applications of private quasi-concave optimization. Given the long line of improvements made for the related task of privately learning thresholds (culminating in Theorem~\ref{thm:IPPintro}), one might guess that similar improvements could be achieved also for private quasi-concave optimization, hopefully reaching error linear or polynomial in $\log^*T$. Surprisingly, we show that this is not the case, and present the following {\em tight} lower bound (up to lower order terms).

\begin{theorem}[Informal version of Theorem~\ref{theo:quasi-concave-lower-bound}]\label{thm:lowerIntro}
Any approximate-private algorithm for quasi-concave optimization must have additive error at least $\widetilde{\Omega}(2^{\log^*T})$, where $T$ is the number of possible solutions.
\end{theorem}

We view this lower bound as having an important conceptual message, because private quasi-concave optimization is the main workhorse (or more precisely, the only known workhorse) for several important tasks, such as privately learning (discrete) halfspaces \cite{BeimelMNS19,DBLP:conf/nips/KaplanMST20}. As such, current bounds on the sample complexity of privately learning halfspaces are exponential in $\log^*|X|$, but it is conceivable that this can be improved to a polynomial or a linear dependency. The lower bound of Theorem~\ref{thm:lowerIntro} means that either this is not true, or that we need to come up with fundamentally new algorithmic tools in order to make progress w.r.t.\ halfspaces.

\subsubsection{The Reorder-Slice-Compute paradigm}
Towards obtaining our upper bounds, we introduce a simple, but powerful, paradigm which we call the {\em Reorder-Slice-Compute (RSC)} paradigm. For presenting this paradigm, let us consider the following algorithm (call it algorithm $\BBB$) that is instantiated on an input dataset $D$, and then for $\tau\in\N$ rounds applies a DP algorithm on a ``slice'' of the dataset.

\vspace{5px}
\noindent\fbox{%
    \parbox{\textwidth}{%
    \vspace{-10px}
\begin{enumerate}[itemsep=0px]
    \item Take an input dataset $D\in X^n$ containing $n$ points from some domain $X$.
    \item For round $i=1,2,\dots,\tau$:
    \begin{enumerate}[topsep=0px]
        \item Obtain an integer $m_i$, an $(\eps,\delta)$-DP algorithm $\AAA_i$ and an ordering $\prec^{(i)}$ over $X$.
        \item\label{step:introm} $S_i\leftarrow$ the largest $m_i$ elements in $D$ under $\prec^{(i)}$.
        \item $D\leftarrow D\setminus S_i$.
        \item $r\leftarrow\AAA(S_i)$.
        \item Output $r$.
    \end{enumerate}
\end{enumerate}
\vspace{-10px}
    }%
}
\vspace{5px}

As $\BBB$ performs a total of $\tau$ applications of $(\eps,\delta)$-DP algorithms, standard composition theorems for DP state that algorithm $\BBB$ itself is $\approx(\eps\sqrt{\tau},\delta\tau)$-DP. This analysis, however, seems wasteful at first glance, because each $\AAA_i$ is applied on a {\em disjoint} portion of the input dataset $D$. That is, the (incorrect) hope here is 
that we do not need to pay in composition
since each data point from $D$ is ``used only
once''. The failure point of this idea is that by deleting {\em one} point from the data, we can create a ``domino effect'' that effects (one by one) many of the sets $S_i$ throughout the execution. This is illustrated in the following example.
\begin{example}\label{example:intro}
Suppose that $X=\N$, and that $m_1=\dots =m_\tau=m$ (for some parameter $m$), and that all of the orderings $\prec^{(1)},\dots,\prec^{(\tau)}$ are the standard ordering of the natural numbers. Now consider the two neighboring datasets $D=\{1,2,3,4,5,...n\}$ and $D'=D\setminus\{1\}$. Then during the execution on $D$ we have that $S_1=\{1,2,\dots,m\}, S_2=\{m+1,\dots,2m\}$, and so on, while during the execution on $D'$ we have that $S'_1=\{2,\dots,m+1\}, S'_2=\{m+2,\dots,2m+1\}$, and so on. That is, even though $D$ and $D'$ differ in only one point, and even though this point is ``used only once'', it generates differences in the output distribution of {\em all} of the iterations, and hence, does not allow us to avoid paying in composition.
\end{example}

A natural idea for trying to tackle this issue, which has been contemplated by several previous papers, is to add noise to the size of each slice \cite{DBLP:conf/focs/BunNSV15,KaplanLMNS20,DBLP:conf/nips/SadigurschiS21}. Specifically, the modification is that in Step~\ref{step:introm} of algorithm $\BBB$ we let $S_i$ denote the largest $(m_i+{\rm Noise})$ elements (for some appropriate noise distribution), instead of the largest $m_i$ elements. The hope is that these noises would ``mask'' the domino effect mentioned above. Indeed, in Example~\ref{example:intro}, if during the first iteration of the execution on $D$ the noise is bigger by one than the corresponding noise during the execution on $D'$, then we would have that only $S_1$ and $S'_1$ differ by one point (the point 1), and after that the two executions continue identically. Thus, the hope is that by correctly ``synchronizing'' the noises between the two executions (such that only the size of the ``correct'' set $S_i$ gets modified by 1), we can make sure that only one iteration is effected, and so we would not need to apply composition arguments.

Although very intuitive, analyzing this idea is not straightforward. The subtle issue here is that it is not clear how to synchronize the noises between the two executions. In fact, this appeared in several papers as an open question.\footnote{
We remark that the analysis of algorithm $\BBB$ (with the noises) becomes significantly easier when all the orderings throughout the execution are the same (as in the setting of Example~\ref{example:intro}). The more general setting (with different orderings) is more challenging, and it is necessary for our applications. We refer the reader to \cite{DBLP:conf/nips/SadigurschiS21} for a more elaborate discussion.} 
Furthermore, this issue (almost) exactly describes the bottleneck in the algorithm of \cite{KaplanLMNS20} for privately learning thresholds, capturing the reason for why their algorithm had sample complexity  $\tilde{O}\left((\log^*|X|)^{1.5}\right)$. We analyze this algorithm, and present the following result.
\begin{theorem}[Informal version of Theorem~\ref{theo:partition-private}]\label{thm:RSCintro}
For every $\hat{\delta}>0$, the RSC paradigm, as described in algorithm $\BBB$ above (with appropriate noises of magnitude $\approx\frac{1}{\eps}$), is $(O(\eps\log(1/\hat{\delta})),\hat{\delta}+2\tau\delta)$-DP.
\end{theorem}

Note that the privacy parameter $\eps$ does {\em not} deteriorate with $\tau$, as it would when using standard composition theorems. This benefit is what, ultimately, allows us to present our improved algorithms for privately learning thresholds and for quasi-concave optimization. % 
As the Reorder-Slice-Compute paradigm is generic, we hope that it would find additional applications in future work.

\subsubsection{A simulation based proof technique} \label{sec:simulation}
Towards analyzing our RSC paradigm, we put forward a new proof technique. 
While obvious in retrospect, and related to prior simulation-based approaches used for proving composition theorems for differential privacy~\cite{KairouzOV15,MurtaghV16}, we believe that our formulation of this proof technique is instructive.

Consider an algorithm $\AAA$ whose input is a dataset, and suppose that we would like to prove that $\AAA$ is DP. To do this, in the proof technique we propose, we design two interactive algorithms: a {\em simulator} $\SSS$ and a {\em data holder} $H$ with the following properties.
The simulator is given two neighboring datasets $D^0$ and $D^1$ but does not know which of these two datasets is the actual input. The task of the simulator is to simulate % 
the computation of $\AAA$ on the actual input dataset $D^b$.
The {\em data holder} has, in addition to $D^0,D^1$,  access to the private bit $b$ (and therefore knows the identity of the actual dataset $D^{b}$). The simulator attempts to perform as much of the computation as they can without accessing the data holder. That is, ideally, the data holder is queried only when it is necessary for a  faithful simulation of $\AAA$ on $D^b$.

The privacy cost of the simulation is with respect to the leakage of the private bit $b$ during the interaction between the simulator and the data holder.  Formally,

\begin{lemma}\label{lemma:intro-simulate}
Let $\AAA$ be an algorithm whose input is a dataset. If there exist a pair of interactive algorithms $\SSS$ and $H$ satisfying the following 2 properties, then algorithm $\AAA$ is $(\eps,\delta)$-DP.
\begin{enumerate}
    \item For every two neighboring datasets $D^0,D^1$ and for every bit $b\in\{0,1\}$ it holds that $$\left(\SSS(D^0,D^1)\leftrightarrow H(D^0,D^1,b)\right)\equiv\AAA(D^b).$$ Here  
    $\left(\SSS(D^0,D^1)\leftrightarrow H(D^0,D^1,b)\right)$
    denotes the outcome of $\SSS$ after interacting with $H$.
    \item Algorithm $H$ is $(\eps,\delta)$-DP w.r.t.\ the input bit $b$.
\end{enumerate}
\end{lemma}

The proof of this lemma is immediate. Nevertheless, embracing its terminology can simplify privacy proofs. The potential benefit comes from the fact that in order to prove that $\AAA$ is DP, we design two other algorithms that are ``working together'' in order to simulate $\AAA$, {\em under the assumption that both of them know the two neighboring datasets}, where $H$ is trying to ``steer'' $\SSS$ towards simulating $\AAA(D^b)$.

Let us elaborate on the benefits of this proof technique in the context of our RSC paradigm (specified in algorithm $\BBB$ above, with noisy slice sizes $m_i$). Fix two neighboring datasets $D,D'$. We design a simulator that, in every iteration $i\in[\tau]$, samples the noisy slice size $m_i$, and checks if the resulting slices $S_i,S'_i$ (corresponding to $D,D'$) are identical. If so, then the simulator does not need to access the data holder, and therefore does not incur a privacy cost. When the simulator encounters a step where $S_i\neq S'_i$, it calls the data holder to perform the computation. When called, in addition to doing the computation, the data holder also attempts to ``synchronize'' the two executions, and reports back to $\SSS$ if it succeeded. Once synchronization is successful, the simulator can proceed without further assistance from the data holder, and no more privacy cost is incurred. We show that, when done correctly, the number of iterations in which we incur a privacy cost is constant in expectation and with probability at least $1-\hat{\delta}$ it is at most $O(\log(1/\hat{\delta}))$.

\subsubsection{Our new upper bound for privately learning thresholds}

To obtain our (nearly tight) upper bound on the sample complexity of privately learning thresholds, we present a new analysis (and a simplification) of the algorithm of \cite{KaplanLMNS20}, which is made possible using our new RSC paradigm. We next survey the algorithm of \cite{KaplanLMNS20} and explain our improvements. We stress that this presentation is oversimplified. Any informalities made herein will be removed in the sections that follow. 

\paragraph{The interior point problem \cite{DBLP:conf/focs/BunNSV15,KaplanLMNS20}.}
Rather than directly designing an algorithm for learning thresholds, the algorithms of \cite{DBLP:conf/focs/BunNSV15,KaplanLMNS20} (as is ours) are stated for the simpler {\em interior point problem}: Given a dataset $D$ containing (unlabeled) elements from an ordered domain $X$, the interior point problem asks for an element
of $X$ between the smallest and largest elements in $D$. Formally,
\begin{definition}\label{def:IPPalg}
An algorithm $\AAA$ solves the interior point problem (IP) over a domain $X$ with sample
complexity $n$ and failure probability $\beta$ if for every dataset $D\in  X^n$,
$$
\Pr[\min D \leq \AAA(D) \leq \max D] \geq 1 - \beta,
$$
where the probability is taken over the coins of $\AAA$. 
\end{definition}

Note that this problem is trivial without privacy constraints (as any input point is a valid output). Nevertheless, solving it with differential privacy has proven to be quite challenging. In particular, 
as Bun et al.~\cite{DBLP:conf/focs/BunNSV15} showed, privately solving this problem is {\em equivalent} to privately learning thresholds (properly).\footnote{
This equivalence is very simple: Given a private algorithm for the IP problem, we can use it to learn thresholds by identifying an interior point of the input points that reside around the decision boundary. For the other direction, given an unlabeled dataset (an instance to the IP problem), sort it, label the first half of the points as 1 and the other half as 0, and use a private algorithm for thresholds in order to identify a decision boundary. This decision boundary is a valid output for the IP problem.} Thus, all of the aforementioned upper and lower bounds w.r.t.\  thresholds apply also to the IP problem, and it suffices to study this simpler problem 
in order to present upper and  lower bounds for privately learning thresholds (properly).

\paragraph{The algorithm of \cite{KaplanLMNS20}.}
Kaplan et al.~\cite{KaplanLMNS20} presented an algorithm, called \texttt{TreeLog}, for privately solving the IP problem.
At a high level, \texttt{TreeLog} works by embedding the input elements from the
domain $X$ in a smaller domain of size $\log|X|$, while guaranteeing that every interior
point of the embedded elements can be (privately) translated into an interior point of the input
elements. The algorithm is then applied recursively to identify an interior point of the embedded
elements. \texttt{TreeLog} can be informally (and inaccurately) described as follows. %

\vspace{10px}
\noindent\fbox{%
    \parbox{\textwidth}{%
{\bf Input:} Dataset $D\in X^n$ containing $n$ points from the ordered domain $X$. 
\begin{enumerate}[topsep=3px,itemsep=0px]

\item Let $T$ be a binary tree with $|X|$ leaves, where every leaf is identified with an element of $X$.

\item\label{step:trim} For a {\em trimming parameter} $t\approx\frac{1}{\eps_0}\log\frac{1}{\delta}$, let $D_{\rm left}$ and $D_{\rm right}$ denote the $t$ smallest and $t$ largest elements in $D$, respectively. Let $\hat{D}=D\setminus(D_{\rm left}\cup D_{\rm right})$.

\item Assign {\em weights} to the nodes of $T$, where the {\em weight} of a node $u$ is the number of input points (from $\hat{D}$) that belong to the subtree of $T$ rooted at $u$.

\item\label{step:path} Identify a path $\pi$ from the root of $T$ to a node $u_{\pi}$ with weight $t$ (in a very particular way).

\item\label{step:embed} Use the path $\pi$ to embed the input points in a domain of size $\log|X|$, where a point $x\in \hat{D}$ is mapped to the {\em level of the tree} $T$ at which it ``falls off'' the path $\pi$. That is, $x$ is mapped to the level of the last node $u$ in $\pi$ s.t.\ $x$ belongs to the subtree rooted at $u$. Points belonging to the subtree rooted at $u_{\pi}$ (the last node in $\pi$) are not embedded (there are $t$ such points). %

\item\label{step:rec} Recursively identify an interior point $\ell^*\in[\log|X|]$ of the $(n-3t)$ embedded points.

\item Let $u^*$ be the node at level $\ell^*$ of $\pi$. Privately choose 
between the left-most and the right-most descendants of $u^*$; one of them is an interior point w.r.t.\ the dataset $D_{\rm left}\cup D_{\rm right}$.

\end{enumerate}
    }%
} 
\vspace{5px}

To see that the algorithm returns an interior point, suppose that (by induction) the point $\ell^*$ from Step~\ref{step:rec} is indeed an interior point of the embedded points. This means that at least one embedded point is smaller than $\ell^*$ and at least one embedded point is larger than $\ell^*$ (for simplicity we ignore here the case where these points might be {\em equal} to $\ell^*$). This means that at least one input point $x_{\rm before}\in\hat{D}$ falls off the path $\pi$ {\em before} level $\ell^*$ and at least one input point $x_{\rm after}\in\hat{D}$ falls off the path $\pi$ {\em after} level $\ell^*$. Observe that since $x_{\rm before}$ falls off $\pi$ before level $\ell^*$, is does {\em not} belong to the subtree rooted at $u^*$ (the node at level $\ell^*$ of $\pi$). On the other hand, $x_{\rm after}$, which falls off $\pi$ after level $\ell^*$, {\em does} belong to the subtree rooted at $u^*$. That is, the subtree rooted at $u^*$ contains some, but not all, of the input points (from $\hat{D}$). Hence, either the left-most descendant of $u^*$, denoted as $u^*_{\rm left\text{-}most}$, or its right-most descendant, $u^*_{\rm right\text{-}most}$, must be an interior point of $\hat{D}$. As $D_{\rm left}\cup D_{\rm right}$ contains $t$ points which are bigger than any point in $\hat{D}$ as well as $t$ points which are smaller than any point in $\hat{D}$, we get that one of $u^*_{\rm left\text{-}most},u^*_{\rm right\text{-}most}$ is a ``deep'' interior point w.r.t.\ $D_{\rm left}\cup D_{\rm right}$ (with at least $t$ points from each side of it). Choosing such a ``deep'' interior point (out of 2 choices) can be done using standard differentially private tools.

The privacy analysis of \texttt{TreeLog} is more challenging. The subtle point is that the path $\pi$ selected in Step~\ref{step:path} is itself {\em highly non private}. Nevertheless, \cite{KaplanLMNS20} showed that \texttt{TreeLog} is differentially private. Informally, the idea is as follows. Fix two neighboring datasets $D$ and $D'=D\cup\{z\}$ and suppose that the same path $\pi$ is selected during both the execution on $D$ and the execution on $D'$. In that case, the embedded datasets generated by the two executions are neighboring, since except for the additional point $z$, all other points are embedded identically. If this is indeed the case, and assuming by induction that \texttt{TreeLog} (with one iteration less) is differentially private, then the recursive call in Step~\ref{step:rec} satisfies privacy. 
The issue is that the path selected by \texttt{TreeLog} is {\em data dependent} and it could be very different during the two executions. Nevertheless, \cite{KaplanLMNS20} showed that when this path is chosen correctly, then it still holds that neighboring datasets are mapped into neighboring embedded datasets,\footnote{More accurately, the distributions on embedded datasets during the two executions are ``close'' in the sense that there is a coupling between neighboring embedded datasets which have similar probability mass.}
which suffices for the privacy analysis. Importantly, for this argument to go through, it is essential that we do {\em not} embed the ``last'' 
$t$ points that fall off the path $\pi$ (the points that belong to the subtree rooted at $u_{\pi}$).

As the domain size reduces logarithmically with each recursive call, after $\log^*|X|$ steps the domain size is constant, and the recursion ends. (This base case, where the domain size is constant, can be handled using standard DP tools.) So there are $\log^*|X|$ steps throughout the execution. Hence, to obtain $(\eps,\delta)$-DP overall, \cite{KaplanLMNS20} applied composition theorems for DP and used a privacy parameter of $\approx\frac{\eps}{\sqrt{\log^*|X|}}$ in every step. This means that we trim $\approx\frac{\sqrt{\log^*|X|}}{\eps}$ points with each iteration (in Steps~\ref{step:trim} and~\ref{step:embed}) and we thus need at least $(\log^*|X|)^{1.5}$ input points in order to make it through to the end of the recursion.

\paragraph{Leveraging the RSC paradigm to obtain our upper bound.} As an application of our RSC paradigm, we present a significantly improved analysis for algorithm \texttt{TreeLog}. Using the terminology of the RSC paradigm, we observe that each iteration of \texttt{TreeLog} cuts out three ``slices'' from the data: Two slices in Step~\ref{step:trim} (the $t$ smallest and $t$ largest elements) and one slice in Step~\ref{step:embed} (the last $t$ points along the path which are not embedded). We show that the algorithm can be written in terms of the RSC paradigm, where every slice is ``used only once''. As a result, we get that it suffices to use a privacy parameter of (roughly) $\eps/\log\frac{1}{\delta}$ for each step of the recursion, while still ending up with $(\eps,\delta)$-DP overall. So now we only trim $\approx\frac{1}{\eps}\log\frac{1}{\delta}$ points in each step, and hence  $\approx\log^*|X|$ points suffice in order to make it till the end of the recursion. 

We stress that this is non-trivial to do without the RSC paradigm. In particular, one of the challenges here is that the embedding used by \texttt{TreeLog} is {\em not} order preserving, and the input points are ``shuffled'' again and again throughout the execution. As a result, there is no a priori order by which we can define the slices throughout the execution. In fact, to make this work, we need to introduce several technical modifications to the \texttt{TreeLog} algorithm, and to generalize the RSC paradigm to support it.

\subsubsection{Another application of the RSC paradigm: axis-aligned rectangles}

We briefly describe another application of our RSC paradigm. Consider the class $C$ of all axis-aligned rectangles over a finite $d$-dimensional grid $X^d\subseteq\reals^d$. A concept in this class could be thought of as the product of $d$ intervals, one on each axis.
Recently, Sadigurschi and Stemmer \cite{DBLP:conf/nips/SadigurschiS21} presented a private learner for this class with sample complexity $\tilde{O}(d\cdot {\rm IP}(X) )$, where ${\rm IP}(X)$ is the sample complexity needed for privately solving the interior point problem over $X$. As a warmup towards presenting their algorithm, \cite{DBLP:conf/nips/SadigurschiS21} considered the following simple algorithm for this problem.

\vspace{10px}
\noindent\fbox{%
    \parbox{\textwidth}{%

{\bf Input:} Dataset $D\in (X^d\times\{0,1\})^n$ containing $n$ labeled points from $X^d$. 

{\bf Tool used:} An algorithm $\AAA$ for the IP problem over $X$ with sample complexity $m$. 
\begin{enumerate}[topsep=3px,itemsep=0px]

\item Let $S\subseteq D$ denote set of all positively labeled points in $D$ (we assume that there are many such points, as otherwise the all-zero hypothesis is a good output).

\item For every axis $i\in[d]$:

\begin{enumerate}[topsep=-5px,itemsep=0px]
    \item Project the points in $S$ onto the $i$th axis.
    
    \item Let $A_i$ and $B_i$ denote the   smallest $(m+{\rm Noise})$ and the largest $(m+{\rm Noise})$ projected points, respectively, without their labels.
    
    \item Let $a_i\leftarrow \AAA(A_i)$ and $b_i\leftarrow\AAA(B_i)$.
    
    \item Delete from $S$ all points (with their labels) corresponding to $A_i$ and $B_i$.

\end{enumerate}

\item Return the axis-aligned rectangle defined by the intervals $[a_i, b_i]$ at the different axes.

\end{enumerate}
    }%
} 
\vspace{5px}

The utility analysis of this algorithm is straightforward. As for the privacy analysis,
observe that there is a total of $2d$ applications of the interior point algorithm $\AAA$ throughout the execution. Hence, using composition theorems, it suffices to run algorithm $\AAA$ with a privacy parameter of roughly $\eps/\sqrt{d}$. However, this would mean that $m$ (the sample complexity of $\AAA$) is at least $\sqrt{d}/\eps$, and hence, each iteration deletes $\approx\sqrt{d}/\eps$ points from the data and we need to begin with  $|S|\gg d^{1.5}/\eps$ input points. So this only results in an algorithm with sample complexity $\tilde{O}\left( d^{1.5}\cdot{\rm IP}(X)\right)$.

To overcome this, \cite{DBLP:conf/nips/SadigurschiS21} designed a more complex algorithm with sample complexity linear in $d$. They left open the possibility that a better analysis of the simple algorithm outlined above could also result in near optimal sample complexity. Indeed, this follows immediately from our RSC paradigm: Every iteration reorders the data points along a different axis, takes out a ``slice'', and computes an interior point of this slice. Hence, by Theorem~\ref{thm:RSCintro}, it suffices to run $\AAA$ with a privacy parameter of $\approx\eps/\log(\frac{1}{\delta})$, which avoids the unnecessary blowup of $\sqrt{d}$ in the sample complexity.

\begin{corollary}
There is an approximate private algorithm for (properly) learning axis-aligned rectangles over a finite $d$-dimensional grid $X^d\subseteq\reals^d$ with sample complexity $\widetilde{O}(d\cdot \log^* |X|)$.
\end{corollary}

\subsubsection{Our results for quasi-concave optimization}

As we mentioned, Bun et al.~\cite{DBLP:conf/focs/BunNSV15} showed that privately learning thresholds is equivalent to privately solving the interior point problem. To obtain our results for quasi-concave optimization, we present a {\em stronger equivalence} in the context of quasi-concave optimization. More specifically, we show that private quasi-concave optimization is equivalent to solving the interior point problem {\em with ``amplified'' privacy guarantees}.\footnote{It is not that the privacy parameters are amplified, rather the resulting algorithm for the interior point problem satisfies a stronger (stringent) variant of differential privacy.} We leverage these amplified privacy guarantees to strengthen the lower bound of \cite{DBLP:conf/focs/BunNSV15} for the interior point problem, thereby obtaining our lower bound of $\Omega(2^{\log^* T})$ for privately optimizing quasi-concave functions. We also leverage this equivalence in the positive direction, and design a suitable variant of our DP algorithm for the IP problem (with ``amplified'' privacy guarantees), thereby obtaining our upper bound of $\tilde{O}(2^{\log^* T})$ for privately optimizing quasi-concave functions.

\subsection{Paper structure}
In Section~\ref{sec:RSCprivacy} we describe and analyze the Reorder-Slice-Compute paradigm.  In Section~\ref{sec:bettertreelog} we present and analyze our algorithm for privately learning thresholds.  In Section~\ref{sec:quasiconcave} we present our results for quasi-concave optimization.

\section{Reorder-Slice-Compute}\label{sec:RSCprivacy}

 In this section we introduce the Reorder-Slice-Compute Paradigm.

\paragraph*{Notation.}  For two reals $a,b\ge 0$, we write $a\approx_\eps b$ if $e^{-\eps} b \le a \le e^\eps b$.
A {\em dataset} $D\in X^n$ can be viewed as a multiset of elements from  $X$: The private algorithms we consider are applied to the respective multiset. 
We refer to an ordered multiset as a \emph{list}.
We consider two multisets or two lists $D,D'$ \emph{adjacent}, if and only if one of them (say, $D'$) can be obtained by inserting a single element into the other.  We say that a
deterministic mapping $E:X^*\to X^*$ from multisets to lists is {\em adjacency preserving}, if for every pair of adjacent data sets $D,D'\cup \{x\}$, $E(D)$ and $E(D')$ are equal or adjacent lists.  To simplify the presentation, we will sometimes treat lists as multisets and apply set operations on both multisets and lists ($D\cup\{x\}$ is the multiset $D$ with the multiplicity of $x$ incremented by $1$).

\subsection{The Reorder-Slice-Compute paradigm}\label{sec:RSC}

\SetKwFunction{RSC}{ReorderSliceCompute}
    
Algorithm~\ref{algo:partition} (\RSC) describes our paradigm. 
The algorithm performs $\tau$ adaptively-chosen computations over disjoint slices of an input dataset $D$. Each computation $i\in[\tau]$ is specified by a tuple
$(m_i, \calA_i, E_i)$:
an $(\varepsilon,\delta)$-DP algorithm $\mathcal{A}_i$, a specified approximate slice size (number of elements) $m_i\in\mathbb{N}$, and an adjacency preserving mapping $E_i:X^*\to X^*$ from data sets to lists.\footnote{One example is where $E_i$ is a sorter that receives the data set $D\in X^n$ and a specified order $\prec$ on $X$ and returns the sorted list of $D$ by $\prec$ (as described in the intro). Our paradigm  allows for more  general data processing than sorting.  This flexibility enables us to express a private algorithm for the interior point problem in this paradigm.}

Given the tuple $(m_i, \calA_i, E_i)$, we use $E_i$ to process the input data set $D_{i-1}$, and select the first $\hat{m}_i := m_i +\Geom(1 - e^{-\eps})$ elements of the list $E_i(D_{i-1})$ into the {\em slice} $S_i$. Then, we apply $\mathcal{A}_i$ to $S_i$, publish the result, and set $D_{i}$ to be the (multiset of the) elements of the list $E_i(D_{i-1})$ with the prefix $S_i$ removed.

The algorithm includes an optional {\em delayed-compute} phase, which follows the slicing phase. The slices $(S_i)_{i=1}^{\tau}$ are kept internally. The algorithm then adaptively receives a slice number $i$ and an $(\eps,\delta)$-DP algorithm $\calA'_i$, and publishes $\calA'_i(S_i)$. Note that each slice is called at most once and the choice of the next slice and the selected algorithm may depend on results from prior slices.

We consider the total privacy cost of \RSC. Intuitively, we might hope for it to be close to 
$(\varepsilon,\delta)$-DP, as each data element contributes to at most one slice. The slices, however, are selected from $D$ in an adaptive and dependent manner. We can bound the total privacy cost using DP composition, but this results in a factor of $\tau$ or $\sqrt{\tau}$ (with advanced composition) increase in the privacy cost. A surprisingly powerful tool is our following theorem that avoids such dependence on $\tau$:

\begin{algorithm2e}[H]
\LinesNumbered
    \caption{Reorder-Slice-Compute (RSC)}
    \label{algo:partition}
    
    \SetKwProg{Fn}{Function}{:}{}
    \SetKwProg{Pg}{Program}{:}{\KwRet}
    
    \DontPrintSemicolon
    
    \SetKwFunction{FSC}{SelectAndCompute}
    
    \KwIn{
         Dataset $D = \{x_1,\dots, x_n\}\in X^n$. Integer $\tau\ge 1$. Privacy parameters $0<\eps,\delta < 1$.
    }
    
    \SetKwProg{Init}{Initialize}{:}{}
    
    \Fn{\FSC{$D$, $m$, $\calA$, $E$}}{
        $\hat{m} \gets m + \Geom(1 - e^{-\eps})$\tcp*{$\Geom(p)$ denotes the geometric distribution with parameter $p$}
        $S\gets $ the first $\hat{m}$ elements in $E(D)$ \;
        $D \gets E(D) \setminus S$\;
        $r \gets \calA(S)$ \;
        \KwRet{$(D,S,r)$}\; %  
    }
    \Pg{}{
\tcp{Slice and Compute Phase:}
        $D_{0}\gets D$ \;
        \For{$i=1,\dots, \tau$}{
            {\bf Receive} $(m_i, \calA_i, E_i)$ where $m_i\in \mathbb{N}$, an $(\eps,\delta)$-DP algorithm $\calA_i$, and an adjacency-preserving mapping $E_i:X^*\to X^*$ from multisets to lists\;
            $(D_i,S_{i}, r_i)\gets $ \FSC{$D_{i-1}$, $m_i$, $\calA_i$, $E_i$ } \;
            \textbf{Publish} $r_i$ \;
        }
\tcp{Delayed Compute Phase:}
    $I\gets [\tau]$\;
    \While{$I$ is not empty}
        {   
            {\bf Receive} $(i, \calA)$, where $i\in I$ and $\calA$ is an $(\eps,\delta)$-DP algorithm\;
            $I \gets I\setminus\{i\}$\;
            \textbf{Publish} $\calA(S_i)$\;
            
        }
    }
\end{algorithm2e}

\begin{theorem}[Privacy of \RSC]\label{theo:partition-private}
For every $\hat{\delta} > 0$, Algorithm~\ref{algo:partition} is $(O(\eps\log(1/\hat{\delta})),\hat{\delta}+2\tau\delta))$-DP.
\end{theorem}

We can consider an extension of \RSC where we allow for up to $k$ compute calls for each slice.  The calls can be made at different points and adaptively, the only requirement is that they are made after the slice is finalized. Our analysis implies the following: 

\begin{corollary} [Privacy of \RSC with $k$  computes per slice]\label{coro:kdelayedcomputeate}\label{coro:kdelayed}
For every $k\geq 1$ and $\hat{\delta} > 0$, an extension of 
Algorithm~\ref{algo:partition} that allows for up to $k$ computations on each slice is
$(O(\eps(k+\log(1/\hat{\delta}))),\hat{\delta}+2k\tau\delta))$-DP.
\end{corollary}

We can also consider performing $k$ adaptive applications of \RSC.
Interestingly, the factor of $\log(1/\hat{\delta})$ loss in privacy is incurred only once:
\begin{theorem}[Privacy of $k$ adaptive applications of \RSC]\label{theo:partition-privatek}
For every $k\geq 1$ and $\hat{\delta} > 0$, 
$k$ adaptive applications of Algorithm~\ref{algo:partition} are
$(O(\eps(k+\log(1/\hat{\delta}))),\hat{\delta}+2k\tau\delta))$-DP.
\end{theorem}

In the following we prove Theorem~\ref{theo:partition-private} (privacy analysis of Algorithm~\ref{algo:partition}).  
We perform the privacy analysis using the simulation-based technique outlined in Section~\ref{sec:simulation}.  In Section~\ref{sec:syncmapping} we
introduce a tool that we call the synchronization mapping, that facilitates the synchronization performed by the data holder. In Section~\ref{sec:sandhdescription} we describe the simulator $\calS$ and data holder $H$ and establish that the simulation faithfully follows Algorithm~\ref{algo:partition}. In Section~\ref{sec:simulprivacy} we show that the data holder satisfies the privacy bounds of Theorem~\ref{theo:partition-private}.
The proof of Theorem~\ref{theo:partition-private} then follows using Lemma~\ref{lemma:intro-simulate}.

The proof of Theorem~\ref{theo:partition-privatek} is a simple extension and is included in Subsection~\ref{sec:privatekproof}.

\subsection{The synchronization mapping}\label{sec:syncmapping}

We specify a pair of randomized mappings $R^b_\varepsilon, b\in \{0,1\}$ that are indexed by a state bit $b$ with the properties described in Lemma~\ref{lemma:sync-map}. 

\paragraph*{Notation.} For a set $S$, $\Delta(S)$ denotes the set of all distributions supported on $S$.
$\Geom(p)$ denotes the geometric distribution with stopping parameter $p$. Formally, $\Pr[\Geom(p) = k] = (1-p)^k\cdot p$ for every $k \ge 0$.

\begin{lemma}[Synchronization lemma]\label{lemma:sync-map}
For every $\eps \in (0, 1)$, there are two randomized mappings $R_\eps^0,R_\eps^1:\mathbb{N}\to \Delta(\mathbb{N}\times \{0,1\})$ such that the following statements hold.
\begin{enumerate}
    \item\label{cond1} For every $m\in \mathbb{N}$, $\supp(R_\eps^0(m)) \subseteq \{(m,0),(m,1)\}$, and $\supp(R_\eps^1(m)) \subseteq \{(m,0),(m-1,1)\}$.
    \item\label{cond2} $R_\eps^0(\Geom(1 - e^{-\eps}))$ and $R_\eps^1(\Geom(1 - e^{-\eps}))$ are $(\eps,0)$-indistinguishable.
    \item\label{cond3} For both $b\in \{0,1\}$, $\Pr_{(\alpha,\beta)\gets R_\eps^b(\Geom(1-e^{-\eps}))}[\beta=1] \ge \frac{1}{6}$.
\end{enumerate}
\end{lemma}

\begin{proof}
We construct a sequence $t_0,\dots, t_{\infty} \in [0,1]^{\mathbb{N}}$ as follows:
\[
t_i = \max\{ 0, e^{-i\eps}+e^{-(i+1)\eps}-1 \}, ~~~\forall i\ge 0.
\]
It is easy to see that $t_i$ is non-increasing and $t_i \le e^{-(i+1)\eps}$. Then, we set
\[
R_\eps^0(0) = \begin{cases}
(0,0) & \text{w.p. $e^{-\eps}$} \\
(0,1) & \text{w.p. $1-e^{-\eps}$} \\
\end{cases} ~~~
\]
and $R_{\eps}^1(0) = (0, 0)$ with probability one. For every $i\ge 1$, we explicitly set
\[
R_{\eps}^0(i) = 
\begin{cases}
(i, 0) & \text{w.p. $t_i \cdot e^{i\eps}$} \\
(i, 1) & \text{w.p. $1 - t_i \cdot e^{i\eps}$} \\
\end{cases},
\]
and
\[
R_{\eps}^1(i) = 
\begin{cases}
(i, 0) & \text{w.p. $t_i \cdot e^{(i+1)\eps}$} \\
(i - 1, 1) & \text{w.p. $1 - t_i \cdot e^{(i+1)\eps}$} \\
\end{cases}.
\]
Note in particular that $\Pr[R^0_\eps(0) = (0,1)] = 1 - e^{-\eps} = 1 - t_0$.

Now we verify the validity of this construction. Obviously $R_{\eps}^0$ and $R_{\eps}^1$ satisfy Property~\ref{cond1}. We verify Property~\ref{cond2} now. We first have
\[
\Pr[R_\eps^0(\Geom(1 - e^{-\eps})) = (0,0)] \approx_{\eps} \Pr[R_\eps^1(\Geom(1 - e^{-\eps})) = (0,0)].
\]
For every $i\ge 1$, we have
\[
\begin{aligned}
&~~~~ \Pr[R_\eps^0(\Geom(1 - e^{-\eps})) = (i,0)] \\
&= (1-e^{-\eps}) e^{-i\eps} \cdot t_i e^{i\eps} \\
&\approx_{\eps} (1-e^{-\eps})e^{-i\eps} \cdot t_i e^{(i+1)\eps} \\
&= \Pr[R_\eps^1(\Geom(1 - e^{-\eps})) = (i,0)].
\end{aligned}
\]

Fix $i\ge 0$ and consider the output $(i,1)$. We have
\[
\begin{aligned}
\frac{\Pr[R_\eps^0(\Geom(1 - e^{-\eps})) = (i,1)]}{\Pr[R_\eps^1(\Geom(1 - e^{-\eps})) = (i,1)]}
&= \frac{e^{-i\eps}(1-e^{-\eps})(1 - t_i e^{i\eps})}{e^{-(i+1)\eps}(1-e^{-\eps})(1 - t_{i+1} e^{(i+2)\eps}) } 
= e^\eps \frac{1 - t_i e^{i\eps}}{1 - t_{i+1} e^{(i+2)\eps}}.  \\
\end{aligned}
\]
Let us consider $1 - t_i e^{i\eps}$. If $t_i = 0$, then $1 - t_i e^{i\eps} = 1$. Otherwise, $1 - t_i e^{i\eps} = e^{i\eps} - e^{-\eps}$. Combining both cases, we conclude that $1 - t_i e^{i\eps} = \min\{1, e^{i\eps}-e^{-\eps}\}$. Similarly, we have
$1 - t_{i+1} e^{(i+2)\eps} = \min\{1, e^{(i+2)\eps}-e^{\eps}\}$. Therefore, it is clear that
\[
e^\eps \frac{1 - t_i e^{i\eps}}{1 - t_{i+1} e^{(i+2)\eps}} =\frac{e^\eps \cdot \min\{1, e^{i\eps}-e^{-\eps}\}}{ \min\{1, e^{(i+2)\eps}-e^{\eps}\}} \in [e^{-\eps}, e^{\eps}].
\]

We have fully verified Property~\ref{cond2}. It remains to verify Property~\ref{cond3}. Let $\gamma\ge 0$ be the minimum integer such that $t_\gamma = 0$. Note that for every input $m\ge \gamma$, with probability one, we have $R^b_{\eps}(m) = (m - b,1)$ for both $b\in \{0,1\}$. Therefore, it is suffices to lower bound $\Pr[\Geom(1-e^{\eps})\ge \gamma] = e^{-\eps \gamma}$. Since $\gamma$ is the minimum integer such that $t_{\gamma} = 0$, we have $t_{\gamma - 1} = e^{-(\gamma - 1)\eps} + e^{-\gamma \eps} - 1 > 0$, implying that $e^{-\eps\gamma} > \frac{1}{1 + e^{\eps}} \ge \frac{1}{6}$ as $\eps \le 1$.
\end{proof}

\subsection{The simulator and data holder} \label{sec:sandhdescription}
We describe the simulator and the data holder and establish that the interaction is a faithful simulation of Algorithm~\ref{algo:partition} and hence satisfies the first condition of 
Lemma~\ref{lemma:intro-simulate}.  To simplify presentation, we present the simulation for Algorithm~\ref{algo:partition} without the delayed compute phase, and then explain how the simulation and analysis can be extended to include delayed compute.

The simulator is described in Algorithm~\ref{algo:simulate} and the data holder query response algorithm is described in Algorithm~\ref{algo:query}.  The simulator receives as input two adjacent datasets 
$D,D' = D\cup \{x\}$.
It then runs a simulation of Algorithm~\ref{algo:partition} and maintains internal state for both cases of the input dataset being $D$ (state $b=0$) and the input dataset being $D'=D\cup \{x\}$ (state $b=1$). The simulation is guaranteed to remain perfect only for the correct case $b$. The simulator initializes $D_0 \gets D$ and updates the active elements $D_i \subset E_{i-1}(D_{i-1})$ and the applicable diff element $x$. The simulator maintains a status bit that is initially $0$ (two cases are not synchronized) and at some point the status becomes and then remains $1$ (two cases are synchronized).
When the status bit is $0$, internal states are maintained  for both cases: The  active elements for case $b=0$ are $D_i$ and the active elements for case $b=1$ are  $D_i$ and one additional element $x$ (initially $x \in D'\setminus D$ but can get replaced). When the status bit is $1$, the internal state is only that of the true case (the active elements of the true case are $D_i$), there is no diff element maintained, and the simulation proceeds like Algorithm~\ref{algo:partition}.

Until synchronization, the simulator slices the data set by emulating $\FSC$.  When the slices are such that they are identical for both cases, that is, the %  
$\hat{m}_i$ prefix of $E_{i-1}(D_{i-1})$ is equal to the $\hat{m}_i$ prefix of $E_{i-1}(D_{i-1}\cup\{x\})$, the computation does not depend on the state $b$ and the simulator performs it and reports the result $r$ without accessing the data holder. The set of active elements with the slice removed continue to differ by the one element $x'$ 
that is the difference of the multisets $E_{i-1}(D_{i-1}\cup\{x\})$ and  $E_{i-1}(D_{i-1})$.
If the slices for the two cases are different, then let $p$ be the first position of the list
$E_{i-1}(D_{i-1}\cup\{x\})$ that does not have the same element as the same position of
$E_{i-1}(D_{i-1})$.  Note that we must have 
$\hat{m}_i \geq \max\{p,m_i\}$.  The slice for the case $b=1$ includes an element $x'$ and the slice for $b=0$ includes a different element $y$ at position $\hat{m}_i$ of $E_{i-1}(D_{i-1})$. The data holder (Algorithm~\ref{algo:query}) therefore must be called to obtain a correct result $r$. The data holder redraws the slice size $\hat{m}_i$ conditioned on it being at least
$\max\{p,m_i\}$. This provides an opportunity for synchronization without changing the distribution (from the memorylessness property of the geometric distribution, the difference under such conditioning is an independent draw from the geometric distribution).  
The data holder attempts to synchronize (that involves applying the randomized mapping that also depends on $b$).  It reports back a triple: The computation result $r$, status indicating whether synchronization was successful, and a slice size $\hat{q}$ to remove from  the prefix of $E_{i-1}(D_{i-1})$ to obtain $D_i$. If there was no synchronization, the simulator computes the new diff element.

Note that if there is no synchronization, the reported size results in perfect removal by the simulator of the elements that participated in the slice for both cases of $b=0$ or $b=1$. The element $y$ that participated in the slice for case $b=0$ but not in $b=1$ replaces $x$.  If the synchronization was successful, then the simulator no longer maintains an additional element and the set $D_i$ is exactly $E_{i-1}(D_{i-1})$ with the elements that participated in the slice for the true case removed.\newline

\newcommand{\status}{\mathsf{status}}
\newcommand{\newstatus}{\mathsf{new\_status}}
\SetKwFunction{FQuery}{Query}

\begin{algorithm2e}[H]
\LinesNumbered
    \caption{The Simulator}
    \label{algo:simulate}
    
    \SetKwProg{Fn}{Function}{:}{}
    \SetKwProg{Pg}{Program}{:}{\KwRet}
    
    \SetKwFunction{FSelect}{Selection}
    \SetKwFunction{FTest}{Test}
    
    \DontPrintSemicolon
    
    \SetKwFunction{FMain}{Main}
    \SetKwFunction{Ftest}{Test}
    
    \KwIn{
        A pair of adjacent datasets $D, D' = D\cup \{x\}$. Integer $\tau\ge 1$. Privacy parameters $\eps\in (0,1),\delta > 0$.
    }
    \SetKwProg{Init}{Initialize}{:}{}
    
    \Pg{}{
        $D_0 \gets D$ \;
        $x\gets x$ \;
        $\status \gets 0$ \tcp*[f]{$\status = 1$ indicates two data sets have been ``synchronized''}\;
        \For{$i=1,\dots, \tau$}{
            Receive $m_i\in \mathbb{N}$, an $(\eps,\delta)$-DP algorithm $\calA_i$, and an adjacency preserving map $E_i:X^*\to X^*$\; 
            \If(\tcp*[f]{Map $E$ eliminated the diff element})
                {$(\status=0)$ $\land$ $(E_i(D_{i-1}) = E_i(D_{i-1}\cup\{x\}))$}
                {$\status \gets 1$\tcp*{Synchronized}}
           
            \If{$\status = 1$}{
                $(D_i,r) \gets $ \FSC{$D_{i-1}$, $m_i$, $\calA_i$, $E_i$ } \;
            }
            \Else{
                $\hat{m}_i \gets m_i + \Geom(1-e^{-\eps})$ \;
                $x' \gets E_i(D_{i-1}\cup\{x\})\setminus E_i(D_{i-1})$\tcp*{diff element of mapped datasets}   
                $p\gets$ the rank of $x'$ in $E_{i-1}(D_{i-1}\cup \{x\})$ \;
                \If(\tcp*[f]{This round does not involve diff element}){$\hat{m}_i < p$}{
                    $S_i \gets $ the first $\hat{m}_i$ elements in $E_i(D_{i-1})$ \tcp*{Slice $S_i$ is the same if selected from $E_i(D_{i-1}\cup\{x\})$}
                    $D_{i} \gets E_i(D_{i-1}) \setminus S_i$ \;
                    $x \gets x'$\;
                    $r \gets \calA_i(S_i)$ \;
                }
                \Else(\tcp*[f]{This round involves diff element}){
                    $q\gets \max(p,m_i)$ \;
                    $(\hat{q}, \newstatus, r)\gets \FQuery(D_{i-1}, x, q, \calA_i, E_i))$ \tcp*{Query the data holder Algorithm~\ref{algo:query} and receive a triple $(\hat{q}, \newstatus, r)\in \mathbb{N}\times \{0,1\}\times \calY$}
                    \If(\tcp*[f]{Synchronization failed}){$\newstatus = 0$}{
                        $y\gets $ the $\hat{q}$-th largest element in $E_i(D_{i-1})$\;
                        $S\gets $ the  first $\hat{q}$ elements in $E_i(D_{i-1})$\;
                        $D_i \gets E_i(D_{i-1}) \setminus S$ \;
                        $x\gets y$ \;
                    }
                    \Else(\tcp*[f]{Successful synchronization}){
                        $S\gets $ the  first $\hat{q}$ elements in $E_i(D_{i-1})$\;
                        $D_i \gets E_i(D_{i-1}) \setminus S$ \;
                        $\status\gets 1$ \;
                    }
                }
            }
            \textbf{Publish} $r$ \;
        }
        \KwRet{$(D_{\tau}, \status, x)$}
    }
\end{algorithm2e}

\begin{algorithm2e}[t]
\LinesNumbered
    \caption{The Query Algorithm to the Data Holder}
    \label{algo:query}
    
    \SetKwProg{Fn}{Function}{:}{}
    \SetKwProg{Pg}{Program}{:}{\KwRet}
    
    \DontPrintSemicolon
    
    \SetKwFunction{FQuery}{Query}
    
    \KwIn{
        A private bit $b\in \{0,1\}$ indicating whether the input data set is $D$ ($b=0$) or $D' = D\cup \{x\}$ ($b=1$). Privacy parameters $\eps\in (0,1),\delta > 0$.
    }
    
    \SetKwProg{Init}{Initialize}{:}{}
    
    \Fn{\FQuery{$D, x, q, \calA, E$}}{
        $\Delta \gets \Geom(1 - e^{-\eps})$ \;
        $\hat{m} \gets q + \Delta$ \;
        \If{$b = 0$}{
            $S\gets $ the first $\hat{m}$ elements in $E(D)$\;
        }
        \Else{
            $S\gets $ the first $\hat{m}$ elements in $E(D \cup\{x\})$\;
        }
        $r \gets \calA(S)$ \;
        $(\alpha, \beta)\gets R^b_\eps(\Delta)$  \tcp*[f]{Try to synchronize}\;
        $\hat{q} \gets q + \alpha$ \tcp*[f]{$\hat{q}$ is the reported number of participating elements.}\;
        \KwRet{$(\hat{q}, \beta, r)$}\;
    }
\end{algorithm2e}

\begin{lemma}
For $b=0$ (resp.~$b=1$), Algorithm~\ref{algo:simulate} simulates the execution of Algorithm~\ref{algo:partition} on the data set $D$ (resp.~$D\cup \{x\}$) perfectly.
\end{lemma}
\begin{proof}
We prove that at the start of each round $i\in [\tau]$, Algorithm~\ref{algo:simulate} maintains the current data set accurately by the triple $(D_{i-1}, x, \status)$, in the following sense.
\begin{itemize}
    \item If $b = 0$, then the current data set is $D_{i-1}$.
    \item Otherwise (i.e., $b=1$), if $\status = 0$, the current data set is $D_{i-1}\cup \{x\}$. If $\status = 1$, the current data set is $D_{i-1}$.
\end{itemize}

We prove the claim by induction on $i\in [\tau]$. This is clearly true for $i = 1$. Now assume the statement holds for $i-1\ge 1$. We prove for the case of $i$. There are three cases:

\smallskip\noindent\textbf{Case 1}.~$\status = 1$. In this case, the current data set is the same for the two cases (that is, is independent of the private bit $b$). Therefore, the call to $\FSC$ is a correct simulation. 

\smallskip\noindent\textbf{Case 2}.~$\status = 0$. In this case, let $p\ge 1$ be the rank of $x$ in $D_{i-1}\cup \{x\}$. To simulate \FSC, we need to sample $\hat{m}_i \gets m_i + \Geom(1 - e^{-\eps})$, and use the first $\hat{m}_i$ elements in the applicable list $E_i(D_{i-1})$ or
$E_i(D_{i-1} \cup\{x\})$ to do the computation. Algorithm~\ref{algo:simulate} first samples $\hat{m}_i$ and tests if $\hat{m}_i < p$. The test yields two cases:
\begin{itemize}
    \item $\hat{m} < p$. In this case, for both $b\in \{0,1\}$, the prefix is the same and Algorithm~\ref{algo:partition} would select the same subset of elements. Therefore, the simulator can perfectly simulate this case without querying the private bit $b$. It is easy to see that the update from $D_{i-1}$ to $D_i$ is valid.
    \item $\hat{m}\ge p$. In this case, the private bit $b\in \{0,1\}$ does make a difference. Hence, the simulator asks the data holder $H$ to do this round of \FSC, \emph{conditioned on} $\hat{m}_i \ge \max(p, m_i)$ (i.e., at least $\max(p,m_i)$ elements are selected). 
    
    We need a well-known fact about Geometric distribution (the memoryless property): suppose there is a random variable $x = a + \Geom(1 - e^{-\eps})$. Then, conditioned on $x \ge y$ for some $y\ge a$, $x$ is distributed as $y + \Geom(1 - e^{-\eps})$. Therefore, conditioned on $\hat{m}_i \ge \max(p, m_i)$, Lines 2-3 in Algorithm~\ref{algo:query} sample  the number of participating elements perfectly.
    
    Having sampled $\hat{m}$, Algorithm~\ref{algo:query} selects the prefix of $\hat{m}$ elements from either $E_i(D_{i-1})$ or $E_i(D_{i-1}\cup \{x\})$ (depending on $b\in \{0,1\}$), and runs $\calA$ on the selected elements. This part simulates Algorithm~\ref{algo:partition} faithfully.
    
    Finally, Algorithm~\ref{algo:query} runs the synchronization mechanism and returns the triple $(\hat{q},\beta, r)$. Given this triple, we verify the validity of the update from $D_{i-1}$ to $D_i$. This is straightforward. 
    
    If $b = 0$, then it is always the case that $\hat{m} = \hat{q}$, and note that Algorithm~\ref{algo:simulate} always removes the first $\hat{m}$ elements from $E_i(D^{i-1})$, no matter what $\newstatus$ is. 
    
    If $b = 1$, then we have $\hat{m} = \hat{q} + \newstatus$. Depending on the value of $\newstatus$, there are two cases: if $\newstatus = 1$, then we update the data set from $D^{i-1} \cup \{x\}$ to $D^i$, where $D^i$ is obtained by removing the first $\hat{q}$ elements from $E_i(D^{i-1})$. Overall this process removes $\hat{q} + 1 = \hat{m}$ elements. If $\newstatus = 0$, then we update the data set from $D_{i-1} \cup \{x\}$ to $D_{i} \cup \{y\}$ where $y$ is the $\hat{q}$-th  element in $E_i(D_{i-1})$. Over all this process removes the first $\hat{q} = \hat{m}$ elements from $E_i(D_{i-1}\cup\{x\})$.
\end{itemize}
In summary, assuming the first $(i-1)$-rounds simulate Algorithm~\ref{algo:partition} perfectly, and the triple $(D^{i-1},x,\status)$ is accurately maintained (in the aforementioned sense), we have shown that the $i$-th round of simulation is also perfect, and the triple is updated accurately. By induction, this shows that the simulator simulates all of the $\tau$ rounds perfectly, as desired.
\end{proof}

\paragraph{Simulation with delayed compute} % 
We outline the modifications needed in the simulation 
when including the delayed compute phase.  It is straightforward to verify that it remains a faithful simulation of \RSC with delayed compute.

The modified simulator performs two phases.  The first is the slicing phase, that is the same as Algorithm~\ref{algo:simulate} except that
the modified simulator stores the slices $S_i$ for steps $i$ that did not require calls to the holder.  It also keeps track of the set of steps $J$ for which it called the data holder.  Additionally, the calls to the data holder also specify the step number $i$.
In the second phase (delayed-compute) the simulator handles $i\not\in J$ by applying the provided algorithm to the stored $S_i$ and publishes the result.  When $i\in J$, the simulator calls the data holder with the specified step number $i$ and the provided algorithm.

The modified data holder (Algorithm~\ref{algo:query}) takes two types of calls, depending on which phase the simulator is in.  The first phase calls correspond to the slicing. These calls are as described in Algorithm~\ref{algo:query}, except that: (1) we allow calls without computations (and results) and (2) the call includes the step number $i$ and the data holder stores internally the applicable slice $S_i$. 

In the delayed-compute phase calls, the input to the data holder is $(i,\calA)$, where $i\not\in J$ is a step number for which it has $S_i$ stored and $\calA$ an $(\eps,\delta)$-DP.  The holder publishes the output $\calA(S_i)$.

\subsection{Simulation privacy analysis}\label{sec:simulprivacy}
The following two lemmas imply that the data holder satisfies the privacy bound stated in Theorem~\ref{theo:partition-private}. They also imply the bound stated in Corollary~\ref{coro:kdelayedcomputeate} for the extension where we allow $k$ $(\eps,\delta)$-DP computations per slice.

\begin{lemma} \label{holdercallprivacy:lemma}
Each call by Algorithm~\ref{algo:simulate} to Algorithm~\ref{algo:query} in the first phase is $(3\eps,2\delta)$-DP and each call in the second phase is $(2\eps,2\delta)$-DP with respect to the private input $b\in \{0,1\}$.
\end{lemma}
\begin{proof}
On a query, the output of Algorithm~\ref{algo:query} is a triple $(\hat{q}, \beta, r)$.  Note that the pair $(\hat{q}, \beta)$ does not depend on the result $r$.  It will be convenient for us to analyse the privacy cost by treating
Algorithm~\ref{algo:query} as first returning $(\hat{q}, \beta)$ and storing $S$, and then at some later point, taking $\calA$ as input and computing and returning $r\gets\calA(S)$.

Note that $(\hat{q},\beta) = (\alpha + q, \beta)$ where $(\alpha, \beta) \sim R^b_\eps(\Geom(1-e^{-\eps}))$. Therefore, by Property~\ref{cond2} in Lemma~\ref{lemma:sync-map}, the pair $(\hat{q},\beta)$ is $(\eps,0)$-DP with respect to the private bit $b$.

The algorithm chooses the set $S$ depending on the private bit $b$. In the following, we use $S^b$ to denote the respective choice. 
Next, having learned $(\hat{q},\beta)$, from the viewpoint of the simulator, she can deduce the following.
\begin{itemize}
    \item If $\beta = 0$, the set $S^b$ used in this query would be the first $\hat{q}$ elements in $E(D)$, or the first $\hat{q}-1$ elements in $D$ plus the extra element $\{x\}$, depending on whether $b$ equals $0$ or $1$. Since $S^0$ and $S^1$ differ by at most two elements, the result $r\sim \calA(S^b)$ is $(2\eps,2\delta)$-DP w.r.t. $b$.
    \item If $\beta = 1$, the set $S^b$ would be the first $\hat{q}$ elements in $E(D)$, or the first $\hat{q}$ elements plus the extra point $\{x\}$, depending on the private bit $b$. Since $S^0$ and $S^1$ differ by at most one element, the result $r\sim \calA(S^b)$ is $(\eps,\delta)$-DP w.r.t. $b$.
\end{itemize}
By the basic composition theorem, $(\hat{q},\beta, r)$ is $(3\eps,2\delta)$-DP w.r.t. the private bit $b$.

Note that this holds also for the delayed-computes that are performed in the second phase and are applied to $S^b$ that are $(2\eps,2\delta)$-DP.
\end{proof}

\begin{lemma}
For every $\hat{\delta} > 0$, with probability $1-\hat{\delta}$, Algorithm~\ref{algo:simulate} makes at most $O(\log(1/\hat{\delta}))$ queries to Algorithm~\ref{algo:query}.
\end{lemma}
\begin{proof}
It suffices to consider the number of calls made during the first phase.
Each time the simulator calls Algorithm~\ref{algo:query}, with probability at least $\frac{1}{6}$, Algorithm~\ref{algo:query} responds a triple with $\beta = 1$. Further observe that once the simulator gets a triple with $\beta = 1$, she will never send query to Algorithm~\ref{algo:query} again. Therefore, the probability that the simulator sends more than $w\in\mathbb{N}$ queries is at most $(5/6)^w$, as desired.  
\end{proof}

\subsection{Analysis for \texorpdfstring{$k$}{k} adaptive applications of Reorder-Slice-Compute} \label{sec:privatekproof}

We outline the proof of Theorem~\ref{theo:partition-privatek}.
We follow the analysis as in the proof of Theorem~\ref{theo:partition-private}.
We apply the simulator for the $k$ executions of \RSC. The privacy cost depends on the total number of calls to Algorithm~\ref{algo:query} which we bound as follows:
\begin{lemma}
Let the random variable $Z_k$ be the number of calls to Algorithm~\ref{algo:query} in $k$ executions of Algorithm~\ref{algo:simulate}.
There is a constant $c$ such that for all $k\geq 1$, $\hat{\delta}>0$,
$\Pr[Z_k \geq c \max\{k,\ln(1/\hat{\delta})\}] \leq \hat{\delta}$.
\end{lemma}
\begin{proof}
The total number of calls to Algorithm~\ref{algo:query} is dominated by the sum of $k$ independent $\Geom(p)$ random variables with parameter $p\geq 1/6$.
Using tail bounds on the sum of Geometric random variables~\cite{JansonTailbounds:2017}, we obtain that for all $\lambda \geq 1$,
\[
\Pr[Z_k \geq \lambda k/p] \leq \exp(-k  (\lambda-1-\ln\lambda)\ .
\]

Substituting $n = \lambda k/p$ we obtain for $n\geq 10 k/p = \Omega(k)$:
$\Pr[Z_k \geq n] \leq \exp(- n/2)$.
Therefore for some constant $c$, for all $\hat{\delta} > 0$,
$\Pr[Z_k \geq c \max\{k,\ln(1/\hat{\delta})\}] \leq \hat{\delta}$.
\end{proof}

\section{Private Learning of Thresholds}\label{sec:bettertreelog}

In this section we describe and analyse an algorithm for the private interior point problem.  Our result for learning thresholds follows from a known connection~\cite{DBLP:conf/focs/BunNSV15} between the two problems.
\mathchardef\mhyphen="2D
\newcommand{\cur}{\mathsf{cur}}
\newcommand{\embd}{\mathrm{embed}}

\subsection{Preliminaries}

We rely on several standard DP mechanisms from the literature. Let us recall the Exponential Mechanism first.

\begin{lemma}[The Exponential Mechanism \cite{DBLP:conf/focs/McSherryT07}]\label{lemma:exponential}
There is an $(\eps,0)$-differentially private algorithm $\calA$ that achieves the following. Let $q:X^*\times Z\to  \mathbb{R}$ be a quality function with sensitivity $1$. Given as input a data set $D\in X^n$, denote $\mathrm{OPT} = \max_{z\in Z}\{ q(D,z) \}$. With probability at least $1-\beta$, $\calA$ outputs a solution $z$ such that $q(D,z)\ge \mathrm{OPT} - \frac{2}{\eps} \log(|Z|/\beta)$.
\end{lemma}

The second mechanism we need is the Choosing Mechanism. We call a quality function $q:X^*\times Z\to \mathbb{R}$ $k$-bounded, if adding a new element to the data set can only increase the score of at most $k$ solutions. Specifically, it holds that
\begin{itemize}
    \item $q(\emptyset, z) = 0$ for every $z\in Z$.
    \item If $D' = D \cup \{x\}$, then $q(D',z) \in \{q(D,z),q(D,z) + 1\}$ for every $z\in Z$, and
    \item There are at most $k$ solutions $z$ such that $q(D',z) = q(D,z) + 1$.
\end{itemize}

The Choosing Mechanism \cite{BeimelNS13} shows how one can privately optimize over $k$-bounded quality functions with improved additive error.

\begin{lemma}[The Choosing Mechanism \cite{BeimelNS13}]\label{lemma:choosing}
Let $\delta > 0$ and $\eps\in (0,2)$. There is an $(\eps,\delta)$-DP algorithm $\calA$ that achieves the following. Let $q : X^*\times Z\to \mathbb{R}$ be a $k$-bounded quality function. Given as input a data set $D\in X^n$, denote $\mathrm{OPT} = \max_{z\in Z}\{q(D,z)\}$. With probability at least $1-\beta$, $\calA$ outputs a solution $z$ such that $q(D,z)\ge \mathrm{OPT} - \frac{16}{\eps} \log(\frac{4kn}{\beta\eps\delta })$.
\end{lemma}

We also need the AboveThreshold algorithm (see, e.g., \cite{DBLP:conf/focs/DworkRV10}).

\begin{lemma}[AboveThreshold \cite{DBLP:journals/fttcs/DworkR14}]\label{lemma:SVT}
There exists an $(\eps,0)$-DP algorithm $\calA$ such that for $m$ rounds, after receiving a sensitivity-$1$ query $f_i:X^*\to \mathbb{R}$, $\calA$ either outputs $\top$ and halts, or outputs $\perp$ and waits for the next round. If $\calA$ was executed with a data set $D\in X^*$ and a threshold parameter $c$, then the following statements hold with probability $1-\beta$:
\begin{itemize}
    \item If a query $f_i$ was answered by $\top$, then $f_i(S) \ge c - \frac{8}{\eps}\log(2m/\beta)$,
    \item If a query $f_i$ was answered by $\perp$ then $f_i(S)\le c + \frac{8}{\eps} \log(2m/\beta)$.
\end{itemize}
\end{lemma}

\subsection{The TreeLog algorithm}

\subsubsection{Setup}

For a totally ordered universe $X = \{ x_1 \prec x_2 \prec \dots \prec x_{|X|} \}$ where the size of $X$ is a power of two (if not, we can append dummy elements to $X$), we build a complete binary $T_X$ with $|X|$ leaves. The $|X|$ leaves are identified with distinct elements in $X$ in order. We call $T_X$ the ``search tree'' for $X$.

For a vertex $v\in T_X$, we use $v_{\textrm{left}}, v_{\textrm{right}}$ to denote the left-most and right-most leaves in the sub-tree rooted at $v$, respectively. We use $v_{\textrm{left}\mhyphen \textrm{right}}$ to denote the right-most leaf in the sub-tree rooted at the left child of $v$.

Given a data set $D\in X^n$, the weighted search tree $T_X^D$ is defined similarly as $T_X$, except that each vertex of $T_X$ is now assigned a weight: every leaf $u$ has weight $w^D(u) := | \{ x\in D : x = u \} |$ and every inner vertex $v$ has weight $w^D(v)$ equal to the sum of the weights of its children. When the data set $D$ is clear from the context, we may omit the superscript and simply write $w(v)$. We also define a score function $f_{\mathrm{IPP}}^D:X\to \mathbb{R}$ with respect to $D$, where $f_{\mathrm{IPP}}^D(z) := \min(|\{x\in D: x\le z\}|, |\{x\in D: x\ge z\}|)$.

\paragraph*{Data processing mappings.} We will design the IPP algorithm based on the RSC framework. To begin with, we define three useful data processing mappings.

We define $E_{\prec}$ and $E_{\succ}$, which sort the data samples. In more detail, $E_{\prec}$ takes as input a data set $D\in X^n$, sorts and outputs the list of samples in $D$ according to the increasing order over $X$. $E_{\succ}$ sorts elements in the decreasing order over $X$, and is defined similarly. We also need an important embedding mapping $E_{\embd}$ to shrink the universe size, which is adapted from \cite{DBLP:conf/nips/KaplanMST20}. 

To describe $E_{\embd}$, let $Y$ be a new universe of size $\log|X|$. For easing the presentation, we identify elements of $Y$ with natural numbers from $1$ through $|Y|$ (i.e., $Y = \{1,2,\dots, \log|X|\}$). The input to $E_{\embd}$ is a data set $D\in X^n$. The output of $E_{\embd}$ is a list of $n$ pairs $E_{\embd}(D)\in (Y\times X)^n$, which are sorted according to the \emph{reversed} lexicographical order.

For a given $D\in X^n$, define
\[
\mathrm{MakeUnlabelledData}(D) = \{(?,x)\in (Y\cup \{?\})\times X : x\in D\}.
\]
Roughly speaking, start from an unlabelled set $D_{new}:=\mathrm{MakeUnlabelledData}(D)$. $E_{\embd}$ gradually assigns labels from $Y$ to data samples in $D_{new}$. After the assignment is done, $E_{\embd}$ sorts $D_{new}$ according to the reversed lexicographical order, and outputs the sorted list. We also need a quantity $\Gamma := \Gamma(D)$ measuring that ``balancedness'' of $D$. Both $E_{\embd}$ and the definition of $\Gamma$ are presented in Algorithm~\ref{algo:select-path}.

\begin{algorithm2e}[H]
\LinesNumbered
    \caption{The $E_{\embd}$ Procedure}
    \label{algo:select-path}
    
    \SetKwProg{Fn}{Function}{:}{}
    \SetKwProg{Pg}{Program}{:}{\KwRet}
    
    \DontPrintSemicolon
    
    \SetKwFunction{FQuery}{Query}
    \SetKwFunction{SelectPath}{$E_{\embd}$}
    
    \KwIn{
        The parameter $t = \frac{100}{\eps}\log(1/\delta)$.
    }
    
    \SetKwProg{Init}{Initialize}{:}{}
    
    \Fn{\SelectPath{$D$}}{
        Construct $T_X^D$ \;
        $\cur\gets $ the root of $T_X^{D}$ \;
        $D_{new}\gets \mathrm{MakeUnlabelledData}(D)$ \\
        $q \gets 1$ \;
        $\Gamma \gets 0$ \;
        \While{\emph{$\cur$ is not a leaf}}{
            $\cur_{\ell},\cur_{r} \gets $ the left and right child of $\cur$ \;
            $\Gamma \gets \max(\Gamma, \min\{w(\cur_{\ell}), w(\cur_{r}) \})$ \;
            \If{$w(\cur_{\ell}) \ge w(\cur_r)$}{
                $(\textrm{next}, \textrm{other}) \gets (\ell, r)$ \;
            }
            \Else{
                $(\textrm{next}, \textrm{other}) \gets (r, \ell)$ \;
            }
            Update $D_{new}$ by assigning ``$q$'' to samples in the subtree rooted at $\cur_{\mathrm{other}}$\;
            $\cur \gets \cur_{\textrm{next}}$ \;
            $q\gets q + 1$ \;
        }
        Define $\Gamma(D):=\Gamma$ \;
        Update $D_{new}$ by assigning ``$|Y|$'' to samples on the leaf $\cur$ \;
        Sort samples in $D_{new}$ in the reversed lexicographical order\; 
        \KwRet{$D_{new}$} \;
    }
\end{algorithm2e}

Algorithm~\ref{algo:select-path} is not always adjacency-preserving. Nevertheless, we have the following observation.

\begin{lemma}\label{lemma:embed-adjacent}
Suppose $D,D' = D\cup \{x\}$ are two adjacent data sets. If $\max(\Gamma(D) ,\Gamma(D')) < t$, then $E_{\embd}(D)$ and $E_{\embd}(D')$ are almost adjacent in the following sense.
\begin{itemize}
    \item There is a way to re-assign the $Y$-label on the first $2t$ samples in the list $E_{\embd}(D)$, such that the modified list is adjacent to $E_{\embd}(D')$.
\end{itemize}
\end{lemma}

\begin{proof}
Let $w$ and $w'$ be the weight functions for $T_X^D$ and $T_X^{D'}$, respectively.

Our first observation is that for every $\cur\in T_X$, as long as $w^0(\cur)> 2t$ (and hence $w^1(\cur) > 2t$), the heavier child of $\cur$ is the same in both $T_X^D$ and $T_X^{D'}$. Therefore, when running Algorithm~\ref{algo:select-path} on $D$ and $D'$, both executions would proceed to the same child. Then, when Algorithm~\ref{algo:select-path} assigns the label $q$, it will assign $q$ to the same branches on both $T_X^D$ and $T_X^{D'}$.

However, Algorithm~\ref{algo:select-path} might finally reach a vertex $\cur$ where the heavier child of $\cur$ is different between $w^0$ and $w^1$. Consequently, in the subtree rooted at $\cur$, the data samples of $D$ and $D'$ might get different labels. Still, we observe that in this case, both children of $\cur$ have roughly equal weights under $w^0$ and $w^1$. Since $\Gamma(D) < t$, this implies that $w^0(\cur) \le 2t$, meaning that there are at most $2t$ samples of $D$ in the subtree rooted at $\cur$, and they will appear as the first $2t$ samples in the resulting list. Thus, one can modify the $Y$-labels of the $2t$ samples so that they agree with the $Y$-labels of the corresponding samples from $D'$.
\end{proof}

We also need the following observation.

\begin{lemma}\label{lemma:gamma-sensitivity-1}
$\Gamma$ is a sensitivity-$1$ function. Namely, $|\Gamma(D) - \Gamma(D')|\le 1$ holds for every pair of adjacent $D,D'$.
\end{lemma}

\paragraph*{Algorithm \texttt{OneHeavyRound}.} Lemma~\ref{lemma:embed-adjacent} shows that $E_{\embd}$ is almost-adjacency-preserving so long as $\Gamma(D)$ is small. We need the following algorithm from \cite{KaplanLMNS20} to deal with the case that $\Gamma(D)$ is large.

\begin{lemma}[adapted from \cite{KaplanLMNS20}]\label{lemma:oneheavyround}
There is an $(O(\eps),O(\delta))$-DP algorithm $\texttt{OneHeavyRound}(X,D)$ that, given a data set $D$ with the promise $\Gamma(D) \ge \frac{t}{2}$, returns an interior point of $D$ with probability at least $1-\delta$.
\end{lemma}

For completeness, we include a proof of Lemma~\ref{lemma:oneheavyround} in Appendix~\ref{append:oneheavy}. Our implementation of $\texttt{OneHeavyRound}$ is simpler than the one described in \cite{KaplanLMNS20}. It also saves an $\log n$-factor in the privacy parameter.

\subsubsection{The algorithm}

We are ready to describe the algorithm.

\begin{algorithm2e}[H]
\LinesNumbered
    \caption{The TreeLog Algorithm}
    \label{algo:threshold}
    
    \SetKwProg{Fn}{Function}{:}{}
    \SetKwProg{Pg}{Program}{:}{\KwRet}

    \DontPrintSemicolon
    
    \SetKwFunction{TreeLog}{TreeLog}
    \SetKwFunction{OneHeavyRound}{OneHeavyRound}
    \SetKwFunction{IPP}{IPP}
    
    \KwIn{
        The privacy parameters $\eps\in (0, 1), \delta > 0$. A trimming parameter $t = \frac{100}{\eps}\log(1/\delta)$.
    }
    
    \SetKwInput{GValue}{Global Variable}
    
    \GValue{Parameter $\rho\in \mathbb{R}$ for AboveThreshold, to be initialized\;}
    
    \Fn{\TreeLog{$X, n, D\in X^n$}}{
        \If{$|X|\le 8$}{
            Use Exponential Mechanism to find $\hat{z}\in X$ according to $f_{\mathrm{IPP}}^D:X \to \mathbb{R}$ \;
            \KwRet{$\hat{z}$} \;
        }
        $(D^{(1)},S_{\ell},\emptyset)\gets $ \FSC{$D$, $t$, $\emptyset$, $E_{\prec}$ } \;
        $(D^{(2)},S_{r},\emptyset)\gets $ \FSC{$D^{(1)}$, $t$, $\emptyset$, $E_{\succ}$ } \;
        $D_{border}\gets S_{\ell} \cup S_{r}$ \;
        \If{$\Gamma(D^{(2)}) + \Lap(1/\eps) \ge \frac{3t}{4} + \rho$}{
            \KwRet{\OneHeavyRound{X,$D^{(2)}$}} \;
        }
        $(D^{(3)},S_{d},\emptyset)\gets $ \FSC{$D^{(2)}$, $2t$, $\emptyset$, $E_{\embd}$ } \;
        $D_{new} \gets \{y : (y,x)\in D^{(3)}\}$ \tcp*{Project $D^{(3)}$ to $Y$}
        $S_{d}\gets \{x:(y,x)\in S_d\}$ \tcp*{Project $S_d$ to $X$}
        \Else{
            $y_{\hat{q}}\gets $\TreeLog{$Y, |D_{new}|, D_{new}$} \tcp*{Solve IPP in the smaller universe.}
            $v\gets $ the $\hat{q}$-th vertex in $\pi$ \;
            $C\gets \{v_{\textrm{left}}, v_{\textrm{right}}, v_{\textrm{left}\mhyphen \textrm{right}}\}$ \\
            Use Choosing Mechanism to find a depth-$\hat{q}$ vertex $v$ in $T_X^{S_d}$ maximizing $w^{S_d}(v)$ \;
            Use Exponential Mechanism to find $\hat{z}\in C$ according to $f_{\mathrm{IPP}}^{D_{border}}:X\to \mathbb{R}$ \;
            \KwRet{$\hat{z}$}
        }
    }
    
    \Pg{\IPP{$X,n,D\in X^n$}}{
        $\rho\gets \Lap(1/\eps)$ \;
        \KwRet{\TreeLog{$X,n,D$}} \;
    }
\end{algorithm2e}

\subsubsection{Utility and privacy analysis}

\paragraph*{Utility analysis.} The utility analysis of our algorithm basically follows from \cite{KaplanLMNS20}. The only additional work is to track the number of samples in the recursive call, as Algorithm~\ref{algo:threshold} randomized the size of slices in each round (compared with the original implementation by \cite{KaplanLMNS20}).

\begin{theorem}[Adapted from \cite{KaplanLMNS20}]\label{theo:threshold-utility}
Algorithm~\ref{algo:threshold} returns an interior point for $D$ with probability at least $1-O(\delta \log^*|X|)$, provided that $n\ge 10\cdot\log^*|X|\cdot  t$.
\end{theorem}

\begin{proof}
When $|X|\le 8$, the claim follows from the utility guarantee of the Exponential Mechanism. In the following, we prove for the case that $\log^*|X|$ is large. 

We first consider the following two types of bad events.
\begin{itemize}
    \item Line~9 is executed with a data set $D^{(2)}$ such that $\Gamma(D^{(2)}) < \frac{1}{2}t$.
    \item The recursion algorithm enters a call $(Y,n,D')$ with $n<10\log^*|Y| \cdot t$.
\end{itemize}

By Lemma~\ref{lemma:SVT}, Event $1$ happens with probability at most $\delta \log^*|X|$. For the second event, suppose the recursion algorithm is currently in the call $(X,n,D)$ with $n\ge 10\log^*|X| \cdot t$. Let $(Y,n',D_{new})$ be the arguments for the next recursive call. Note that 
\[
n' = n - 4t - \Geom(1-e^{-\eps}) - \Geom(1-e^{-\eps}) - \Geom(1-e^{-\eps}).
\]
Therefore, it follows that $n'\ge 10 \log^*|Y| \cdot t$ with probability at least $1-\delta$. Since there are at most $\log^*|X|$ rounds of recursion, the probability that Event 2 happens is at most $\delta \log^*|X|$.

In the following, we condition on that neither of the two events happens, and show that Algorithm~\ref{algo:threshold} succeeds in finding an interior point with probability $1-O(\delta \log^*|X|)$. 

First, at the corner case of the recursion (i.e., Line~4 or 9), Algorithm~\ref{algo:threshold} returns an interior point with probability at least $1-\delta$ by the utility guarantee of the Exponential Mechanism and Choosing Mechanism (Lemmas~\ref{lemma:exponential} and \ref{lemma:choosing}).

Now we consider the case that involves recursion. In each level of the recursion $(X,n,D)$, we condition on that the recursive call $(Y,|D_{new}|,D_{new})$ returns an interior point $y_{\hat{q}}$ for $D_{new}$. Then, there is a depth-$\hat{q}$ vertex $v$ in $T_X$ such that, all the samples of $S_d$ lie in the subtree rooted at $v$. Since $|S_d|\ge 2t$, by Lemma~\ref{lemma:choosing}, Line~15 succeeds in finding the vertex $v$ with probability at least $1-\delta$. Next, we claim that one of $v_{\mathrm{left}}, v_{\mathrm{right}},v_{\mathrm{left}\mhyphen\mathrm{right}}$ is an interior point for $D^{(2)}$. There are two cases to verify.
\begin{itemize}
    \item All the samples of $D^{(2)}$ are in the subtree rooted at $v$. Then, since $y_{\hat{q}}$ is an interior point for $D_{new}$, at least one sample gets assigned label $\hat{q}$ when we run $E_{\embd}(D^{(2)})$. This means that both children of $v$ contain at least one sample, implying that $v_{\mathrm{left}\mhyphen\mathrm{right}}$ is an interior point.
    \item At least one sample of $D^{(2)}$ is not in the subtree rooted at $v$. Then, since the subtree of $v$ contains at least one sample, we conclude that at least one of $v_{\mathrm{left}}, v_{\mathrm{right}}$ is an interior point w.r.t. $D^{(2)}$.
\end{itemize}
Now, by the construction of $D_{border}$, any interior point of $D^{(2)}$ has quality score at least $t$ under $f_{\mathrm{IPP}}^{D_{border}}$. By Lemma~\ref{lemma:exponential}, with probability $1-\delta$, Line~18 finds a vertex $\hat{z}$ with quality score at least $\frac{t}{2}$ under $f_{\mathrm{IPP}}^{D_{border}}$. This implies that $\hat{z}$ is an interior point of $D$, as desired.

To summarize, assuming the recursive call returns an IP correctly, Lines 14-19 succeeds in finding an IP with probability at least $1-O(\delta)$. Since there are at most $\log^*|X|$ rounds of recursion, the probability that we fail to find an IP for the original data set is at most $O(\delta \log^*|X|)$. This calculation is conditioned on that neither of two aforementioned bad events happens, which holds with probability $1-O(\delta \log^*|X|)$. Overall, we conclude that Algorithm~\ref{algo:threshold} finds an interior point with probability $1-O(\delta \log^*|X|)$.
\end{proof}

\paragraph*{Privacy analysis.} We analyze the privacy property of Algorithm~\ref{algo:threshold}.

\begin{theorem}\label{theo:threshold-privacy}
Algorithm~\ref{algo:threshold} is $(O(\eps\log(1/\delta)),O(\delta \log^*|X|))$-DP.
\end{theorem}

\begin{proof}
Note that Algorithm~\ref{algo:threshold} can be seen as a \emph{concurrent} composition\footnote{Roughly speaking, concurrent composition means running several interactive DP mechanisms in parallel, where the queries to different mechanisms can be arbitrarily interleaved. See, e.g., \cite{DBLP:conf/tcc/VadhanW21, VadhanZ2022, Lyu22}.} of three pieces of algorithms:
\begin{enumerate}
    \item The AboveThreshold algorithm.
    \item The Reorder-Slice-Compute paradigm.
    \item Algorithm $\OneHeavyRound$.
\end{enumerate}
Items~1 and 3 are easy to see from the code of Algorithm~\ref{algo:threshold}. Item~2 is trickier: the issue is that $E_{\embd}$ is not always adjacency-preserving. Still, Lemma~\ref{lemma:embed-adjacent} shows that if we restrict the inputs to $E_{\embd}$ to data sets $D$ such that $\Gamma(D) < t$, then $E_{\embd}$ is ``almost'' adjacency-preserving. Further note that after running Line~10 ($(D^{(3)},S_d,\emptyset) \gets$\FSC{$D^{(2)}$, $2t$, $\emptyset$, $E_{\embd}$}), the algorithm will project $S_{d}$ to $X$ (i.e., it will ignore the $Y$-labels of $S_d$) before using it to do any private computation. Hence, ``almost''-adjacency-preserving is sufficient for us to analyze the algorithm via the Reorder-Slice-Compute paradigm.

Note that the execution of Algorithm~\ref{algo:threshold} might become completely non-private once the recursion attempts to call $E_{\embd}$ with a data set $D$ such that $\Gamma(D) \ge t$. By the utility guarantee of SVT, this event happens with probability at most $\delta \log^*|X|$. We condition on that this event does NOT happen. Then, Item~1 is $(O(\eps), 0)$-DP, Item~$2$ is $(O(\eps\log(1/\delta)),\delta \log^*|X| )$-DP, and Items~$3$ is $(O(\eps), O(\delta))$-DP. Overall, we conclude that Algorithm~\ref{algo:threshold} is $(O(\eps\log(1/\delta)),O(\delta \log^*|X|))$-DP, as desired.
\end{proof}

\subsection{Near optimal learning of thresholds}

We are ready to prove Theorem~\ref{thm:IPPintro}. Suppose we aim for an $(\eps,\delta)$-DP algorithm for solving the interior point problem where $\eps,\delta \in (0, 1)$. Consider using Algorithm~\ref{algo:threshold} with privacy parameters $\eps' = \frac{\eps}{C\log(1/\delta)}$ and $\delta' = \delta^C$ for a large enough constant $C$. The trimming parameter would be $t = \frac{100}{\eps'} \log(1/\delta')\le O(\frac{1}{\eps}\log^2(1/\delta))$. Let $n = 10t\log^*|X| \le O(\frac{\log^2(1/\delta)\log^*|X|}{\eps})$. It follows from Theorems~\ref{theo:threshold-utility} and \ref{theo:threshold-privacy} that Algorithm~\ref{algo:threshold} is $(\eps,\delta)$-DP. Meanwhile, it solves the interior point problem with sample complexity $n$ and success probability $1-\delta$.

By the known connections between the task of privately learning thresholds and the interior point problem \cite{DBLP:conf/focs/BunNSV15}, we establish the following theorem, which is the formal version of Theorem~\ref{thm:IPPintro}.

\begin{theorem}\label{theo:learning-threshold}
For any privacy parameters $\eps,\delta\in (0,1)$, any finite and totally ordered domain $X$, any desired utility parameters $\xi, \beta\in (0,1)$, there is a sample size 
\[
n\le O\left(\frac{\log^*|X|\cdot \log^2(\frac{\log^*|X|}{\beta\delta})}{\xi \eps}\right)
\]
and an $(\eps,\delta)$-DP algorithm $\calA:(X\times \{0,1\})^n\to X$ that PAC-learns thresholds over $X$ within generalization error $\xi$ with probability at least $1 - \beta$ in the realizable setting (when there is a threshold function that is consistent with the data).
\end{theorem}

Via known realizable-to-agnostic transformations \cite{BeimelNS21ALGORITHMICA,AlonBMS20}, Theorem~\ref{theo:learning-threshold} also implies a result for the agnostic setting:

\begin{corollary}
For any privacy parameters $\eps,\delta\in (0,1)$, any finite and totally ordered domain $X$, any desired utility parameters $\xi, \beta\in (0,1)$, there is a sample size 
\[
n\le O\left(\frac{\log^*|X|\cdot \log^2(\frac{\log^*|X|}{\beta\delta})}{\xi \eps}
+
\frac{\log(\frac{1}{\beta})}{\xi^2 \eps}
\right)
\]
and an $(\eps,\delta)$-DP algorithm $\calA:(X\times \{0,1\})^n\to X$ that PAC-learns thresholds over $X$ within generalization error $\xi$ with probability at least $1 - \beta$ in the agnostic setting.
\end{corollary}

\section{Private Quasi-Concave Optimization}\label{sec:quasiconcave}

In this section, we present our results for private quasi-concave optimization. In particular, we prove nearly matching upper and lower bounds for the achievable additive error of quasi-concave optimization under privacy constraints.

\subsection{Cumulatively-DP}

The key concept in this section is a privacy definition that we call ``cumulatively-DP''. This is similar to the standard definition of $(\eps,\delta)$-DP but with a relaxed requirement of ``adjacency'' between data sets. Formally, it is defined as follows.

\begin{definition}
Let $X$ be an ordered domain. Let $D,D'\in X^*$ be two data sets of the same size. We say that $D$ and $D'$ are cumulatively adjacent, if for every $y\in X$, it holds that
\[
\big| |\{x\in D: x\le y\}| - |\{x\in D':x\le y\}| \big| \le 1
\]
The name ``cumulative'' arises because we are requiring that the ``cumulative counts'' of $D$ and $D'$ at any threshold $y$ differ by at most $1$. 

More generally, we say $D,D'$ are of cumulative distance $d\in \mathbb{N}$, if the inequality above holds with integer $d$ on the right hand side.

An algorithm $\calA:X^n\to \calY$ is called $(\eps,\delta)$-cumulatively DP, if $\calA(D)$ and $\calA(D')$ are $(\eps,\delta)$-indistinguishable (satisfy the $(\eps,\delta)$-DP requirement of Definition~\ref{def:DP}) for every pair of \emph{cumulatively} adjacent data sets $D$, $D'$.
\end{definition}

\paragraph*{Connections to quasi-concave optimization.} We exploit a close connection between cumulatively-DP algorithms and private algorithms for quasi-concave optimization. 

More precisely, let $X$ be an ordered domain. For a given data set $D\in X^n$, we define a score function $f_{\mathrm{IPP}}:X^n\times X\to \mathbb{R}$ as
\[
f_{\mathrm{IPP}}(D,y) = \min\{ |\{x\in D: x\le y\}|, |\{x\in D:x\ge y\}| \}.
\]
It is easy to see that $f$ is quasi-concave. Moreover, if $D$ and $D'$ are cumulatively adjacent, then $\|f_{\mathrm{IPP}}(D,\cdot ) - f_{\mathrm{IPP}}(D', \cdot) \|_{\infty}\le 1$. Therefore, $f_{\mathrm{IPP}}$ is of sensitivity-$1$ under the cumulative adjacent definition.

Note that solving the interior point problem for $D$ amounts to find a solution $y$ such that $f(D,y) > 0$. Also note that $\max_{y}\{f_{\mathrm{IPP}}(D,y)\} = \frac{|D|}{2}$. Hence, we obtain the following connection.
\begin{itemize}
    \item If there is an $(\eps,\delta)$-DP algorithm for sensitivity-$1$ quasi-concave optimization over $X$ with additive error $n$, then there is an $(\eps,\delta)$-cumulatively DP algorithm for the interior point problem over $X$ with sample complexity $2n+1$.
\end{itemize}
Therefore, to prove an error lower bound for private quasi-concave optimization, it suffices to prove a sample complexity lower bound for cumulatively-DP IPP algorithms, which is done in Section~\ref{sec:quasiconcave-lowerbounds}.

\paragraph*{The upper bound connection.} Next, we show how to translate a sample complexity \emph{upper bound} for IPP to a private quasi-concave optimization algorithm with small additive error. Suppose there is an $(\eps,\delta)$-cumulatively DP algorithm for the interior point problem over $X$ with sample complexity $n$. Then one can design an $(O(\eps), O(\delta))$-DP algorithm for private quasi-concave optimization over $X$ with additive error $O(n)$.

We describe the algorithm now. Fix an arbitrary sensitivity-$1$ quasi-concave score function $f:\calD\times X\to \mathbb{N}$\footnote{We assume for simplicity that the scores are always non-negative integers. It is not hard to generalize our algorithm to handle the more general case of real scores.}. Given a data set $D$, we first calculate $\Delta(f, D) = \max_{y\in X}\{f(D,y)\} - \min_{y\in X}\{f(D,y)\}$. One can use the Laplace mechanism to distinguish between $\Delta(f, D)\le n$ and $\Delta(f,D)> 2n$. If $\Delta(f, D)\le 2n$, any solution would be at most $2n$-far from optimal. In this case, the algorithm simply returns an arbitrary element of $X$. In the following, we assume that $\Delta(f, D)\ge n$.

Now, let $\mathrm{OPT} = \mathrm{OPT}(D) = \max_{y\in X}\{f(D,y)\}$. Define a function $f'(D,\cdot)$ as $f'(D,\cdot) = \max\{0, f(D,y) - \mathrm{OPT} + n\}$. Then, finding a good solution for $f(D,\cdot)$ reduces to finding a solution for $f'(D,\cdot)$ with non-zero score.

Next, construct a data set $S\in X^{n}$ from $f'(D,\cdot)$ as follows. For brevity, we represent $X$ as $X = \{1,2,\dots, |X|\}$, and set $f'(D,0) = f'(D,|X|+1) = 0$. Then, for every $y\in [1,|X|]$, we add $\max\{f'(D,y) - f'(D,y-1),0\}$ copies of $y$ into $S$. It is easy to verify that the number of elements in $S$ is exactly $n$. We can use the assumed cumulatively-DP IPP algorithm to find an interior point $y$ for $S$. It it easy to verify that $f'(D,y) > 0$. This completes the description, as well as the utility analysis of the algorithm.

To see the privacy, note that for every adjacent data sets $D,D'$ (for the quasi-concave optimization problem), we have $\| f(D,\cdot) - f(D',\cdot) \|_\infty\le 1$. If we construct $f'(D,\cdot)$ and $f'(D',\cdot)$ in the way described above, we have $\| f'(D,\cdot) - f'(D',\cdot)\|_{\infty} \le 2$. Let $S,S'$ be the two data sets constructed from $f'(D,\cdot)$ and $f'(D',\cdot)$. One can verify that $S,S'$ are of cumulative distance at most $2$. Since we have assumed that the IPP algorithm is cumulatively DP, the privacy of the quasi-concave optimization algorithm follows.

In Section~\ref{sec:quasiconcave-upperbounds}, we design a cumulatively-DP algorithm for the interior point problem with sample complexity $\widetilde{O}(2^{\log^*|X|})$, thus proving Theorem~\ref{thm:upperIntro}.

\subsection{Upper bounds}\label{sec:quasiconcave-upperbounds}

In this section, we state and prove the sample complexity upper bound for solving IPP under the cumulatively-DP constraint.

\begin{theorem}\label{theo:quasi-concave-upper-bound}
Let $X$ be a finite totally ordered domain. For every $\eps,\delta \in (0, 1)$, there is a sample size 
\[
n\le O\left(2^{\log^*|X|}\frac{\log^*|X|(\log^*|X| + \log(1/\delta))}{\eps}\right)
\]
and an $(\eps,\delta)$-cumulatively DP algorithm $\calA:X^n\to X$ that solves the interior point problem over $X$ with probability $\frac{9}{10}$.

As a corollary, there is an $(\eps,\delta)$-DP algorithm for the quasi-concave optimization problem with additive error $\widetilde{O}(\frac{2^{\log^*|X|}}{\eps})$.
\end{theorem}

We prove Theorem~\ref{theo:quasi-concave-upper-bound} by modifying Algorithm~\ref{algo:threshold}. In the following, we describe several key ideas in the modification, and explain why we end up with sample complexity $\widetilde{O}(2^{\log^*|X|})$. The formal description of the algorithm and analysis will be deferred to Appendix~\ref{append:quasi-concave}.

Recall that Algorithm~\ref{algo:threshold} works under the Reorder-Slice-Compute paradigm with three data processing mappings $E_{\succ}, E_{\prec}$ and $E_{\embd}$. Here, the key property allowing us to establish the privacy guarantee is that all of $E_{\succ},E_{\prec}$ and $E_{\embd}$ preserve (or almost-preserve) adjacency.

To generalize Algorithm~\ref{algo:threshold} to a cumulatively-DP algorithm, let $D,D'\in X^n$ be two cumulatively adjacent data sets. It is easy to see that $E_{\succ}(D)$ and $E_{\succ}(D')$ are cumulatively-adjacent. Similarly, $E_{\prec}$ also preserves cumulative adjacency. It remains to consider $E_{\embd}$. In fact, we can establish an analogue of Lemma~\ref{lemma:embed-adjacent}, as follows.

\begin{lemma}\label{lemma:embed-cumu-adjacent}
Suppose $D,D'\in X^n$ are two data sets of cumulative distance $d$. If $\max(\Gamma(D) ,\Gamma(D')) < t - 2d$, then $E_{\embd}(D)$ and $E_{\embd}(D')$ are close in the following sense.
\begin{itemize}
    \item There is a way to re-assign the $Y$-label to the first $2t$ samples in the list $E_{\embd}(D)$ to get a modified list $\tilde{E}_{\embd}(D)$, which satisfies the following condition.
    \item Suppose we project $\tilde{E}_{\embd}(D)$ and $E_{\embd}(D')$ to the $Y$-coordinate. Then, they are of cumulative distance at most $2d$.
\end{itemize}
\end{lemma}

The proof of Lemma~\ref{lemma:embed-cumu-adjacent} is deferred to Appendix~\ref{append:quasi-concave}. Intuitively, Lemma~\ref{lemma:embed-cumu-adjacent} shows that although $E_{\embd}$ does not preserve cumulative-adjacency, it can at most double the cumulative distance between data sets, as long as the inputs $D,D'$ have reasonably small $\Gamma$ values.

\paragraph*{Algorithm overview.} We are ready to sketch the algorithm. We design the algorithm following the framework of Algorithm~\ref{algo:threshold}, with the following modifications.
\begin{itemize}
    \item We no longer need the improved RSC paradigm to save the $\log^*|X|$ factor in the privacy bound. Therefore, when creating the slices, we can deterministically specify their sizes.
    \item Suppose we desire the final algorithm to be $(O(\eps),O(\delta))$-DP. Then, we need to set $\eps' = \frac{\eps}{2^{\log^*|X|}}$, and use $\eps'$ as the privacy parameter to run the underlying DP components (i.e., the AboveThreshold algorithm, the Choosing Mechanism and Exponential Mechanism, etc.).
\end{itemize}

We are able to show that the whole algorithm is $(O(\eps),O(\delta))$-DP. Intuitively, the reason is that the recursive algorithm includes three data processing mappings $E_{\succ},E_{\prec}$ and $E_{\embd}$. We have shown that both $E_{\succ}$ and $E_{\prec}$ preserve cumulative adjacency, while each application of $E_{\embd}$ can at most double the cumulative distance. Observe that there are at most $\log^*|X|$ levels of recursion. Starting from a pair of cumulatively-adjacent data sets $D,D'$, at the end of the recursion, the distance between the remaining data can be at most $2^{\log^*|X|}$. Then, applying an $(\eps',0)$-DP computation on the remaining data translates to an $O(\eps)$-privacy loss of the whole algorithm.

Since we use $\eps'$ as the privacy parameter for each component of the algorithm, to get a good utility guarantee, we have to set $t \approx \frac{1}{\eps'} = \frac{2^{\log^*|X|}}{\eps}$. Finally, this translates to the sample complexity upper bound of $\widetilde{O}(\frac{2^{\log^*|X|}}{\eps})$.

\subsection{Lower bounds}\label{sec:quasiconcave-lowerbounds}

We prove the lower bound in this subsection. That is, we prove

\begin{theorem}\label{theo:quasi-concave-lower-bound}
Let $\eps = 0.2$. There is a constant $C > 1$ such that the following is true. For every sufficiently large domain $X$ and every $n\le 2^{\log^*|X|-C \log^*\log^*|X|}$, no $(\eps,\frac{1}{50})$-cumulatively DP algorithm $\calA:X^n\to X$ can solve the interior point problem over $X$ with probability $\frac{9}{10}$. 

As a corollary, every $(\eps,\delta)$-DP algorithm for quasi-concave optimization over $X$ must incur an additive error of $\Omega(2^{\log^*|X|})$.
\end{theorem}

The rest of the subsection is devoted to proving Theorem~\ref{theo:quasi-concave-lower-bound}. The proof a careful modification of the argument from \cite{DBLP:conf/focs/BunNSV15}.

Define a sequence $(B_i)_{i\ge 1}$ as $B_i = 10i^2$. We define a sequence of universe as follows.
\[
\begin{aligned}
\calX_{1} &= [B_1] \\
\calX_{i} &= [B_{i}]^{\calX_{i-1}}, ~~~ i\ge 2
\end{aligned}
\]
We use lexicographical order on each $\calX_i$. That is, assuming $\calX_{i-1}$ has been totally ordered, then each element of $\calX_{i}$ can be written as a mapping $a:\calX_{i-1}\to [B_{i}]$. For two elements $a,b\in \calX_{i}$, we say that $a < b$ if there is $x\in \calX_{i-1}$ such that $a(x) < b(x)$, and $a(y) = b(y)$ for every $y<x, y\in\calX_{i-1}$. 

It was shown in \cite{DBLP:conf/focs/BunNSV15} that $\log^*|\calX_i| \le i + C\log^*i$ for a constant $C > 0$. Finally, to prove Theorem~\ref{theo:quasi-concave-lower-bound}, it suffices to prove the following lemma.

\begin{lemma}\label{lemma:lb-induction}
For every $m\ge 1$, letting $n = 2^{m-1}$, there is no $(\eps, \frac{1}{50})$-cumulatively DP algorithm $\calA: \calX_m^{n}\to \calX_m$ that can solve the interior point problem over $\calX_m$ with probability more than $\frac{1}{2} - \frac{1}{3m}$.
\end{lemma}

We prove Lemma~\ref{lemma:lb-induction} by induction on $m$. The case of $m = 1$ is straightforward. Given only $1$ data point, there is no algorithm to solve IPP over $\calX_1 := [10]$ satisfying both $(\eps, \frac{1}{50})$-cumulatively DP and error probability (over worst-case inputs) less than $\frac{5}{6}$.

Now we assume that the lemma is true for $m-1$. We prove it for the case of $m$. Suppose, for the sake of contradiction, that there is an $(\eps,\frac{1}{50})$-cumulatively DP algorithm $\calA$ for IPP over $\calX_{m}$ with sample complexity $2^{m}$ and success probability $\frac{1}{2} - \frac{1}{3m}$. We consider the following attempt to solve IPP over $\calX_{m-1}$ that used the assumed algorithm as a black box.

\begin{algorithm2e}[H]
\LinesNumbered
    \caption{The Reduction Algorithm}
    \label{algo:reduction-lb}
    
    \SetKwProg{Fn}{Function}{:}{}
    \SetKwProg{Pg}{Program}{:}{\KwRet}
    
    \DontPrintSemicolon
    
    \SetKwFunction{IPP}{IPPReduction}
    \SetKwFunction{SelectPath}{SelectPath}
    
    \KwIn{
        $n=2^{m-1}$; A data set $(x_1,x_2,\dots, x_n) \in \calX_{m-1}^{n} $
    }
    
    \SetKwProg{Init}{Initialize}{:}{}
    
    \Fn{\IPP{}}{
        $z\gets_R [2,B_{m}-1]^{\calX_{m-1}}$ \;
        \For{$i=1,\dots, n$}{
            Construct $y_i^{0}:\calX_{m-1}\to [B_{m}]$ as
            $y_i^{0}(v) = \begin{cases}
            z(v) & \text{if $v\le x_i$;} \\
            1 & \text{otherwise.}
            \end{cases}
            $\;
            Construct $y_i^{1}:\calX_{m-1}\to [B_{m}]$ as
            $y_i^{1}(v) = \begin{cases}
            z(v) & \text{if $v\le x_i$;} \\
            B_m & \text{otherwise.}
            \end{cases}
            $\;
        }
        $y^*\gets \calA(\{y_i^{b}:i\in [n], b\in \{0,1\}\})$ \;
        Let $\ell\in \calX_{m-1}$ be the largest element such that $y^*(v)=z(v), \forall v\le \ell$ \;
        \KwRet{$\ell$}\;
    }
\end{algorithm2e}

Assuming that there is no private algorithm for IPP over $\calX_{m-1}$ with sample complexity $2^{m-1}$ and error probability $\frac{1}{2}+\frac{1}{3(m-1)}$, the following two claims rule out the possibility of an algorithm $\calA$ with the aforementioned privacy and utility guarantees.

\begin{claim}
Assuming $\calA$ is $(\eps, \delta)$-cumulatively DP, Algorithm~\ref{algo:reduction-lb} is $(\eps, \delta)$-cumulatively DP.
\end{claim}

\begin{proof}
Suppose $D = \{x_1,\dots, x_n\}$ and $D' = \{x'_1,\dots, x'_n\}$ are a pair of cumulatively adjacent data sets. We compare the execution of Algorithm~\ref{algo:reduction-lb} on the input $D$ and $D'$. Fix an arbitrary $z\in [2,B_{i}-1]^{\calX_{m-1}}$. Consider the executions of Lines 3-5 in Algorithm~\ref{algo:reduction-lb} on $D,D'$. Denote the obtained data sets as $U = \{ y_{i}^{b}:i\in [n],b\in \{0,1\} \}$, $U'= \{ {y'}_{i}^{b}:i\in [n],b\in \{0,1\} \}$. It is easy to verify that $U,U'$ are cumulatively adjacent. Given that $\calA$ is $(\eps,\delta)$-cumulatively DP, the output $y^*$ in Line~6 is $(\eps,\delta)$-indistinguishable between the two executions with $D$ and $D'$. Finally, as Lines~6 and 7 do not depend on the private data, we conclude that $\ell$ is $(\eps,\delta)$-indistinguishable between the two executions, which implies that Algorithm~\ref{algo:reduction-lb} is $(\eps,\delta)$-cumulatively DP.
\end{proof}

\begin{claim}
Suppose $\calA$ solves IPP over $\calX_{m}$ with probability at least $\frac{1}{2} - \frac{1}{3m}$. Then Algorithm~\ref{algo:reduction-lb} solves IPP over $\calX_{m-1}$ with probability at least $\frac{1}{2} - \frac{1}{3(m-1)}$. 
\end{claim}

\begin{proof}
Fix an arbitrary input $D = \{x_1,\dots, x_n\}$. Let $\underline{x}$, $\bar{x}$ be the minimum and maximum element in $D$, respectively. Then, it is easy to see that the data set $\{y_i^b\}$ constructed in Algorithm~\ref{algo:reduction-lb} never depends on $z(v)$ for $v>\bar{x}$. Therefore, one can equivalently implement Algorithm~\ref{algo:reduction-lb} in the following way.
\begin{enumerate}
    \item Sample $z(v)\gets_R [2,B_m-1]$ for every $v\le \bar{x}$.
    \item Construct $\{y_i^b\}$ and get $y^*\gets \calA(\{y_i^b\})$.
    \item Sample $z(v)\gets_R [2,B_m-1]$ for every $v > \bar{x}$.
    \item Calculate and return $\ell$ as in Lines 7-8.
\end{enumerate}
After finishing the first two steps, we condition on the event that $y^*$ solves the IPP instance of $\{y_i^b\}$, which happens with probability at least $\frac{1}{2}-\frac{1}{3m}$ by the claim assumption. Then, it must be the case that $y^*(v) = z(v)$ for every $v\le \underline{x}$, as otherwise all elements of $\{y_i^b\}$ would lie on only one side of $y^*$. Let $\bar{x}_{next}$ be the successor of $\bar{x}$ in $\calX_m$. When we sample $z(\bar{x}_{next})$, the probability that $y^*(\bar{x}_{next}) = z(\bar{x}_{next})$ is at most $\frac{1}{B_m-2}$. If this event does NOT happen, $\ell$ must be sandwiched between $\underline{x}$ and $\bar{x}$, with means that it solves the IPP instance $D$. Overall, the error probability is at most
\[
\left(\frac{1}{2} + \frac{1}{3m}\right) + \frac{1}{B_m-2} \le \frac{1}{2} + \frac{1}{3(m-1)},
\]
which completes the proof.
\end{proof}

\begin{small}
\paragraph*{Acknowledgements}
Edith Cohen is partially supported by Israel Science Foundation (grant no. 1595/19).
Uri Stemmer is partially supported by the Israel Science Foundation (grant 1871/19) and by Len Blavatnik and the Blavatnik Family foundation.
\end{small}

\bibliographystyle{alpha}

\newcommand{\etalchar}[1]{$^{#1}$}

\appendix

\section{The OneHeavyRound Algorithm}\label{append:oneheavy}

In this section, we present our simplified implementation of \texttt{OneHeavyRound}. Recall that this is an $(O(\eps),O(\delta))$-DP algorithm that, given an input $D$ with the promise $\Gamma(D) \ge \frac{t}{2}$, finds an interior point for $D$ with probability $1-O(\delta)$.

\begin{algorithm2e}[H]\label{algo:oneheavy}
\LinesNumbered
    \caption{$\texttt{OneHeavyRound}$}
    \label{algo:select-path-app}
    
    \SetKwProg{Fn}{Function}{:}{}
    \SetKwProg{Pg}{Program}{:}{\KwRet}
    
    \DontPrintSemicolon
    
    \SetKwFunction{FQuery}{Query}
    \SetKwFunction{SelectPath}{OneHeavyRound}
    
    \KwIn{
        The parameter $t = \frac{100}{\eps}\log(1/\delta)$.
    }
    
    \SetKwProg{Init}{Initialize}{:}{}
    
    \Fn{\SelectPath{$D$}}{
        Construct $T_X^D$ \;
        $\cur\gets $ the root of $T_X^{D}$ \;
        $\rho\gets \Lap(1/\eps)$ \;
        \While{$\cur$ is not a leaf}{
            $\cur_{\ell},\cur_{r} \gets $ the left and right child of $\cur$ \;
            $w_{\min}(\cur) \gets \min\{w(\cur_\ell),w(\cur_r)\}$ \;
            \If{$w_{\min}(\cur) > \frac{t}{10}$ and $w_{\min}(\cur)+\Lap(1/\eps) \ge \frac{t}{4} + \rho$}{
                \KwRet{$\cur_{\mathrm{left}\mhyphen\mathrm{right}}$}
            }
            \If{$w(\cur_{\ell}) \ge w(\cur_r)$}{
                $(\textrm{next}, \textrm{other}) \gets (\ell, r)$ \;
            }
            \Else{
                $(\textrm{next}, \textrm{other}) \gets (r, \ell)$ \;
            }
            $\cur \gets \cur_{\textrm{next}}$ \;
        }
        \KwRet{$\cur$}\;
    }
\end{algorithm2e}

We prove the privacy and utility of Algorithm~\ref{algo:oneheavy}.

\begin{lemma}
Suppose we restrict the input to Algorithm~\ref{algo:oneheavy} to data sets $D$ such that $\Gamma(D) \ge \frac{t}{2}$. Then, Algorithm~\ref{algo:oneheavy} is $(4\eps,O(\delta))$-DP.
\end{lemma}

\begin{proof}
Fix $D,D' = D\cup \{x\}$ to be a pair of adjacent data sets such that $\Gamma(D),\Gamma(D') \ge \frac{t}{2}$. Let $w^0,w^1$ be the weight functions w.r.t. $T_X^D$ and $T_X^{D'}$, respectively.

Imagine running Algorithm~\ref{algo:oneheavy} on $D$ or $D'$ in parallel. Algorithm~\ref{algo:oneheavy} will maintain two pointers $\cur, \cur'$. In each step, the algorithm moves $\cur$, $\cur'$ to the heavier child w.r.t. $T_X^D$ and $T_X^{D'}$, respectively. One can imagine that $\cur$ and $\cur'$ will move forward to the same child for a while before they finally deviate. Let $u$ be the \emph{last} vertex such that $\cur$ and $\cur'$ are about to deviate.

We consider the following two types of bad events (with respect to both $D$ and $D'$).
\begin{itemize}
    \item After processing the vertex $u$, Algorithm~\ref{algo:oneheavy} still does not return.
    \item Algorithm~\ref{algo:oneheavy} returns $v_{\mathrm{left}\mhyphen\mathrm{right}}$ at a vertex $v$ such that only one of $w_{\min}^0(v)$, $w_{\min}^1(v)$ is larger than $\frac{t}{10}$.
\end{itemize}

Since $u$ is the last vertex before deviation, there is a vertex $u'$ that is on the path from the root to $u$ (inclusive), such that $\min(w^0_{\min}(u'),w^{1}_{\min}(u')) \ge \frac{t}{2}$. To see this, consider the following two cases: 
\begin{itemize}
    \item If $w^0(u) \ge t$, then $w^0_{\min}(u) \ge \frac{t}{2}$, because the heavier child of $u$ is different w.r.t. $w^0$ and $w^1$.
    \item If $w^0(u) < t$, then any vertex in the subtree rooted at $u$ cannot witness that $\Gamma(D) \ge \frac{t}{2}$. Hence, the witness for $\Gamma(D)$ must appear before $u$, which means that there exists a vertex $u'$ on the path from the root to $u$ such that $w_{\min}^0(u')\ge \frac{t}{2}$.
\end{itemize}

Now, by the concentration property of the Laplace noise, the probability that Algorithm~\ref{algo:oneheavy} does not return before or at the vertex $u'$ is at most $O(\delta)$.

For Item~2, note that there is at most one vertex $v$ satisfying the condition. Moreover, it must be the case that $w_{\min}^0(v) = 0$ and $w_{\min}^1(v) = 1$. Hence, the probability that Algorithm~\ref{algo:oneheavy} returns $v_{\mathrm{left}\mhyphen\mathrm{right}}$ is at most $O(\delta)$.

Now, note that if neither of the two events happens, Algorithm~\ref{algo:oneheavy} can be seen as an instantiation of the AboveThreshold algorithm, which is known to be $(4\eps,0)$-DP. Overall, we conclude that Algorithm~\ref{algo:oneheavy} is $O(4\eps,O(\delta))$-DP, as desired.
\end{proof}

\begin{remark}\label{rem:heavy-cumulative}
Algorithm~\ref{algo:oneheavy} also satisfies the stronger requirement of cumulatively-DP (see Section~\ref{sec:quasiconcave}). Indeed, suppose $D,D'\in X^n$ are two data sets of cumulative distance $d$, such that $d < \frac{t}{20}$ and $\min(\Gamma(D),\Gamma(D'))\ge \frac{t}{2}$. Then, using a similar argument as above, one can show that the output distributions of Algorithm~\ref{algo:oneheavy} on $D,D'$ are $(O(d\eps), O(\frac{n\delta}{t}))$-indistinguishable. The factor $\frac{n}{t}$ arises because there are at most $O(\frac{n}{t})$ vertices $\cur$ where only one of $w_{\min}^0(\cur),w^1_{\min}(\cur)$ is larger than $\frac{t}{10}$. This observation is important for designing the cumulatively-DP IPP algorithm (see Appendix~\ref{append:quasi-concave}).
\end{remark}

We finish this section by proving the utility of Algortithm~\ref{algo:oneheavy}.

\begin{lemma}
Given a data set $D$ with the promise that $\Gamma(D) \ge \frac{t}{2}$, Algorithm~\ref{algo:oneheavy} returns an interior point of $D$ with probability at least $1-\delta$.
\end{lemma}

\begin{proof}
Note that as long as Algorithm~\ref{algo:oneheavy} returns at Line~$9$, the returned point $\cur_{\mathrm{left}\mhyphen\mathrm{right}}$ must be an interior point (because both children of $\cur$ have non-zero weights). Suppose $u$ is the vertex that witnesses $\Gamma(D)$ (namely, $w_{\min}(u) \ge \frac{t}{2}$). Then, even if we ignore all the vertices before $u$, Algorithm~\ref{algo:oneheavy} returns $u_{\mathrm{left}\mhyphen\mathrm{right}}$ with probability at least $1-\delta$. Since testing more vertices can only make the probability larger, we conclude that Algorithm~\ref{algo:oneheavy} succeeds in finding an interior point with probability at least $1-\delta$.
\end{proof}

\section{Cumulatively-DP Interior Point Algorithm}\label{append:quasi-concave}

In this section, we present the formal proof of Theorem~\ref{theo:quasi-concave-upper-bound}. We start with the proof of Lemma~\ref{lemma:embed-cumu-adjacent}. Recall its statement.

\begin{reminder}{Lemma~\ref{lemma:embed-cumu-adjacent}}
Suppose $D,D'\in X^n$ are two data sets of cumulative distance $d$. If $\max(\Gamma(D) ,\Gamma(D')) < t - 2d$, then $E_{\embd}(D)$ and $E_{\embd}(D')$ are close in the following sense.
\begin{itemize}
    \item There is a way to re-assign the $Y$-label on the first $2t$ samples in the list $E_{\embd}(D)$ to get a modified list $\tilde{E}_{\embd}(D)$, which satisfies the following condition.
    \item Suppose we project $\tilde{E}_{\embd}(D)$ and $E_{\embd}(D')$ to the $Y$-coordinate. Then, they are of cumulative distance at most $2d$.
\end{itemize}
\end{reminder}

\begin{proof}
Let $w^0,w^1$ be the weight functions w.r.t. $T_X^D$ and $T_X^{D'}$, respectively. The first part of the proof is similar to that of Lemma~\ref{lemma:embed-adjacent}. Namely, let $u$ be the last vertex $\cur$ such that $E_{\embd}(D)$ and $E_{\embd}(D')$ are about to deviate after processing $u$. Then, it must be the case that $w^0(u)\le 2t$. Suppose otherwise. Then, we have $w^0_{\min}(u)\le \Gamma(D)$. Consequently, $|w^{0}(u_{\ell}) - w^0(u_{r})| \ge 2t+1 - 2 \Gamma(D) > 4d$. Since we have assumed that $D$ and $D'$ are of cumulative distance $d$, it is impossible for the heavier children of $u$ to be different in $T_X^D$ and $T_X^{D'}$. Hence, we conclude that $w^0(u) \le 2t$, meaning that there are at most $2t$ samples unassigned after deviation.

It remains to verify that the cumulative distance of the new pair of data sets at most doubles before the deviation. Indeed, for every depth $q$ before deviation, let $c_q$ be the $q$-th vertex on the path from the root to $u$. We observe that the set of samples receiving a label $y \le q$ is exactly the set of samples that are \emph{not} in the subtree rooted at $c_q$. Since $D,D'$ are of cumulative distance at most $d$, we know that $|w^0(c_q)-w^1(c_q)|\le 2d$. Therefore, the counts differ by at most $2d$ on $D$ and $D'$. Note that this argument holds for every $q$ before deviation. Also, after the deviation, we can modify the $Y$-labels of the samples in $D$ arbitrarily. Overall, we conclude that there is a way to modify the first $2t$ samples in $E_{\embd}(D)$, so that the modified list is of cumulative distance $2d$ from $E_{\embd}(D')$ (after projecting to the $Y$ coordinate).
\end{proof}

\paragraph*{The algorithm.} We are ready to describe the algorithm.

\begin{algorithm2e}[H]
\LinesNumbered
    \caption{The Cumulative TreeLog Algorithm}
    \label{algo:quasi-concave}
    
    \SetKwProg{Fn}{Function}{:}{}
    \SetKwProg{Pg}{Program}{:}{\KwRet}

    \DontPrintSemicolon
    
    \SetKwFunction{TreeLog}{NewTreeLog}
    \SetKwFunction{Slice}{Slice}
    \SetKwFunction{Main}{CUMULATIVE-IPP}
    
    \KwIn{
        The privacy parameters $\eps\in (0, 1), \delta > 0$. A trimming parameter $t = \frac{100}{\eps}\log(1/\delta)$.
    }
    
    \SetKwInput{GValue}{Global Variable}
    
    \GValue{Parameter $\rho\in \mathbb{R}$ for AboveThreshold, to be initialized\;}
    
    \Fn(\tcp*[h]{Slicing data with deterministic slice size}){\Slice{$D$, $m$, $E$}}{
        $S\gets $ the first $m$ elements in $E(D)$ \;
        $D \gets E(D) \setminus S$\;
        \KwRet{$(D,S)$}\; % 
    }
    
    \Fn{\TreeLog{$X, n, D\in X^n$}}{
        \If{$|X|\le 8$}{
            Use Exponential Mechanism to find $\hat{z}\in X$ according to $f_{\mathrm{IPP}}^D:X \to \mathbb{R}$ \;
            \KwRet{$\hat{z}$} \;
        }
        $(D^{(1)},S_{\ell})\gets $ \Slice{$D$, $t$, $E_{\prec}$ } \;
        $(D^{(2)},S_{r})\gets $ \Slice{$D^{(1)}$, $t$, $E_{\succ}$ } \;
        $D_{border}\gets S_{\ell} \cup S_{r}$ \;
        \If{$\Gamma(D^{(2)}) + \Lap(1/\eps) \ge \frac{3t}{4} + \rho$}{
            \KwRet{\OneHeavyRound{X,$D^{(2)}$}} \;
        }
        $(D^{(3)},S_{d})\gets $ \Slice{$D^{(2)}$, $2t$, $E_{\embd}$ } \;
        $D_{new} \gets \{y : (y,x)\in D^{(3)}\}$ \tcp*{Project $D^{(3)}$ to $Y$}
        $S_{d}\gets \{x:(y,x)\in S_d\}$ \tcp*{Project $S_d$ to $X$}
        \Else{
            $y_{\hat{q}}\gets $\TreeLog{$Y, |D_{new}|, D_{new}$} \tcp*{Solve IPP in the smaller universe.}
            $v\gets $ the $\hat{q}$-th vertex in $\pi$ \;
            $C\gets \{v_{\textrm{left}}, v_{\textrm{right}}, v_{\textrm{left}\mhyphen \textrm{right}}\}$ \\
            Use Choosing Mechanism to find a depth-$\hat{q}$ vertex $v$ in $T_X^{S_d}$ maximizing $w^{S_d}(v)$ \;
            Use Exponential Mechanism to find $\hat{z}\in C$ according to $f_{\mathrm{IPP}}^{D_{border}}:X\to \mathbb{R}$ \;
            \KwRet{$\hat{z}$}
        }
    }
    
    \Pg{\Main{$X,n,D\in X^n$}}{
        $\rho\gets \Lap(1/\eps)$ \;
        \KwRet{\TreeLog{$X,n,D$}}
    }
\end{algorithm2e}

The utility of Algorithm~\ref{algo:quasi-concave} follows from that of Algorithm~\ref{algo:threshold}, which we omit for brevity. We establish the privacy claim now.

\begin{theorem}\label{theo:cumu-IPP-privacy}
For every $\eps < 2^{-\log^*|X|}$, Algorithm~\ref{algo:quasi-concave} is $(O(\eps 2^{\log^*|X|}), O(\delta 2^{\log^*|X|}))$-cumulatively-DP.
\end{theorem}

\begin{proof}
Note that Algorithm~\ref{algo:quasi-concave} can be seen as a concurrent composition of three pieces of algorithms.
\begin{enumerate}
    \item The AboveThreshold algorithm (where each query has sensitivity at most $2^{\log^*|X|}$).
    \item The Reorder-Slice-Compute paradigm (without randomizing the slice size).
    \item The $\OneHeavyRound$ algorithm.
\end{enumerate}

Let $D,D'$ be any pair of cumulatively-adjacent data sets. Consider the executions of Algorithm~\ref{algo:quasi-concave} on $D$ and $D'$. We condition on the event that Algorithm~\ref{algo:quasi-concave} never calls $E_{\embd}$ with a data set $B$ such that $\Gamma(B) \ge \frac{9}{10}t$. By the utility guarantee of SVT, this event happens with probability at least $1 - \delta \log^*|X|$. Note that $t \ge 20\cdot  2^{\log^*|X|}$. Therefore, this event implies that every applications of $E_{\embd}$ at most doubles the cumulative distance of the inputs.

Since there are at most $\log^*|X|$ levels of recursion. By Lemma~\ref{lemma:embed-cumu-adjacent}, at each level of the recursion, queries to the AboveThreshold algorithm and the $\OneHeavyRound$ algorithm have sensitivity at most $2^{\log^*|X|}$. Hence, Item~1 is $(O(\eps 2^{\log^*|X|}), 0)$-DP, and Item~3 is $(O(\eps 2^{\log^*|X|} ),O(\frac{\delta n}{t}))$-DP (see Remark~\ref{rem:heavy-cumulative}). In our parameter setting, we always have $n \le t\cdot \poly(\log^*|X|)$. Hence, this is $(O(\eps 2^{\log^*|X|} ),O(\delta \poly(\log^*|X|)))$-DP.

Now we consider Item~$2$. Note that the $\log^*|X|$ levels of recursion create at most $3\log^*|X|$ slices. Suppose the original data sets are $D,D'$, which are of cumulative distance $1$. Then, for each $i\in [\log^*|X|]$, at the $i$-th level of the recursion, the produced slices are of cumulative distance $2^{i-1}$: the claim is trivial for $S_{\ell}, S_r$ and $D_{new}$. For $S_d$, the claim is also true because of the operation of projecting onto $X$. Note that Algorithm~\ref{algo:quasi-concave} runs at most one $(O(\eps),O(\delta))$-cumulatively-DP algorithm on each of these slices. By the group privacy property and the basic composition of DP, we conclude that the whole computation is $(\overline{\eps}, \overline{\delta})$-DP, where
\[
\overline{\eps} = O\left( \eps \sum_{i=1}^{\log^*|X|} 2^{i-1} \right) = O(\eps 2^{\log^*|X|})
\]
and similarly $\bar{\delta} = O(\delta 2^{\log^*|X|})$.

Overall, we conclude that Algorithm~\ref{algo:threshold} is $(O(\eps 2^{\log^*|X|}),O(\delta 2^{\log^*|X|}))$-DP, as desired.
\end{proof}

\paragraph*{Proof of Theorem~\ref{theo:quasi-concave-upper-bound}.} We are ready to prove Theorem~\ref{theo:quasi-concave-upper-bound}. Suppose we aim for an $(\eps,\delta)$-cumulatively-DP algorithm for solving the interior point problem where $\eps,\delta \in (0, 1)$. Consider using Algorithm~\ref{algo:quasi-concave} with privacy parameters $\eps' = \frac{\eps}{C 2^{\log^*|X|}}$ and $\delta' = \delta^C\cdot 2^{-\log^*|X|}$ for a large enough constant $C$. The trimming parameter would be $t = \frac{100}{\eps'} \log(1/\delta')\le O(2^{\log^*|X|}\frac{\log^*|X| + \log(1/\delta)}{\eps})$. Let $n = 10t\log^*|X| \le O(2^{\log^*|X|}\frac{\log^*|X|(\log^*|X| + \log(1/\delta))}{\eps})$. Finally, we know that Algorithm~\ref{algo:quasi-concave} is $(\eps,\delta)$-cumulatively-DP. Furthermore, it solves the interior point problem with sample complexity $n$ and success probability at least $1-\delta$.

\end{document}